\documentclass{article}
\usepackage[utf8]{inputenc}
\usepackage[T1]{fontenc}
\usepackage[preprint, main]{neurips_2026}
\usepackage{hyperref}
\usepackage{url}
\usepackage{booktabs}
\usepackage{amsfonts}
\usepackage{nicefrac}
\usepackage{microtype}
\usepackage{xcolor}
\usepackage{amsmath, amssymb, amsthm}
\usepackage{graphicx}
\usepackage{algorithm}
\usepackage{algorithmic}
\usepackage{float}
\usepackage{bbm}
\setcitestyle{authoryear,round,citesep={;},aysep={,},yysep={;}}
\usepackage[most]{tcolorbox}
\definecolor{studentbg}{RGB}{255,246,235}
\definecolor{teacherbg}{RGB}{238,252,240}
\newtcolorbox{methodbox}[1][]{enhanced, breakable,
  colback=white, colframe=black!25,
  boxrule=0.6pt, arc=2.5pt,
  borderline west={2.6pt}{0pt}{teal!70!black},
  left=9pt, right=9pt, top=7pt, bottom=7pt,
  shadow={0.9pt}{-0.9pt}{0pt}{black!10},
  fontupper=\normalsize,
  colbacktitle=teal!10!white, coltitle=teal!80!black,
  fonttitle=\bfseries\sffamily,
  titlerule=0pt,
  title={#1},
  before skip=9pt, after skip=9pt}

\newtcolorbox{resultbox}[1][]{enhanced, breakable,
  colback=white, colframe=black!25,
  boxrule=0.6pt, arc=2.5pt,
  borderline west={2.6pt}{0pt}{purple!70!black},
  left=9pt, right=9pt, top=7pt, bottom=7pt,
  shadow={0.9pt}{-0.9pt}{0pt}{black!10},
  fontupper=\normalsize,
  colbacktitle=purple!10!white, coltitle=purple!80!black,
  fonttitle=\bfseries\sffamily,
  titlerule=0pt,
  title={#1},
  before skip=9pt, after skip=9pt}

\newtheorem{theorem}{Theorem}
\newtheorem{lemma}[theorem]{Lemma}
\newtheorem{proposition}[theorem]{Proposition}
\newtheorem{corollary}[theorem]{Corollary}
\newtheorem{definition}{Definition}
\newtheorem{assumption}{Assumption}
\newtheorem{remark}{Remark}

\newcommand{\mstd}[2]{#1\if\relax\detokenize{#2}\relax\else{$\scriptstyle\,\pm\,#2$}\fi\%}

\title{PACED: Distillation and On-Policy Self-Distillation at the Frontier of Student Competence}

\author{
    Yuanda Xu\thanks{Equal contribution.}\,
    \thanks{Correspondence to \texttt{yuanda@math.princeton.edu}} \quad
    Hejian Sang\footnotemark[1] \quad
    Zhengze Zhou\footnotemark[1] \quad
    Ran He\footnotemark[1] \quad
    Zhipeng Wang \\
    LinkedIn Corporation \\
    \texttt{\{yuanda@math.princeton.edu, hejian@alumni.iastate.edu, zz433@cornell.edu,} \\
    \texttt{rh2528@columbia.edu, 
    zhipeng.wang@alumni.rice.edu\}}
}

\begin{document}

\maketitle

\begin{abstract}

Standard LLM distillation treats all training problems equally---wasting compute on problems the student has already mastered or cannot yet solve.
We empirically show that this inefficiency has a precise gradient-level signature: the cross-problem gradient signal-to-noise ratio (SNR) follows a bell curve over student pass rate, collapsing at both extremes.

We propose \textsc{Paced}, which weights each problem by $w(p) = p(1{-}p)$ where $p$ is the student's empirical pass rate---concentrating training on the \emph{zone of proximal development}.
This requires only student rollouts, no architectural changes, and no hyperparameters.
We prove the Beta kernel $w(p) = p^\alpha(1{-}p)^\beta$ is the leading-order optimal weight family arising from the SNR boundary-collapse structure, and is minimax-robust under misspecification (worst-case efficiency loss $O(\delta^2)$).

Across Qwen3, Qwen2.5, and Llama-3 families, \textsc{Paced} sets a new state of the art in our experimental setting on MATH-500, AIME~2024, and AIME~2025, improving over unweighted distillation by up to $\mathbf{+8.2}$ and over the strong AKL baseline by up to $\mathbf{+3.6}$, while reducing forgetting to $\mathbf{1.4\%}$ and $\mathbf{0.6\%}$ in distillation and self-distillation.
A two-stage forward-then-reverse KL schedule pushes gains further to $\mathbf{+5.8}$ over standard forward KL on the hardest benchmark.
\end{abstract}

\section{Introduction}
\label{sec:intro}
Knowledge distillation trains a student model to imitate a teacher, yet standard practice spreads the training budget \emph{uniformly} across all problems. This is wasteful: some problems are already mastered and provide redundant signal, while others are far beyond the student's current reach and produce noisy, incoherent gradients that can erode previously learned knowledge~\citep{french1999catastrophic, kirkpatrick2017ewc}.

We begin with an empirical observation that makes this intuition precise. To our knowledge, we are the \textbf{first to directly measure the cross-problem gradient signal-to-noise ratio (SNR) in distillation as a function of student competence}. The measurement reveals a striking pattern (Figure~\ref{fig:snr_empirical}): the SNR follows a \textbf{bell-shaped curve that collapses at both boundaries}. At $p \approx 0$, gradients from diverse intractable problems are directionally incoherent. At $p \approx 1$, per-problem gradients persist (the teacher's distribution may remain sharper) but point in problem-specific directions that \emph{disperse} in parameter space, canceling when averaged. The maximum SNR---where each gradient step most efficiently advances the student's capability frontier---lies at intermediate pass rates.

This observation motivates \textbf{P}roficiency-\textbf{A}daptive \textbf{C}ompetence \textbf{E}nhanced \textbf{D}istillation (\textsc{Paced}): weight each problem by $w(p) = p(1{-}p)$, concentrating training on the \emph{zone of proximal development}~\citep{vygotsky1978mind}---problems where the student sometimes succeeds and sometimes fails. The weighting requires only student rollouts, no architectural changes, and no hyperparameters.

Why this specific functional form? Taking the boundary-collapse structure as a regularity condition, we show that any SNR profile with power-law decay at both boundaries decomposes as $p^{a'}(1{-}p)^{b'} \cdot e^{r(p)}$ with bounded remainder (Proposition~\ref{prop:representation}), making the Beta kernel $w(p) = p^\alpha(1{-}p)^\beta$ the leading-order, maximum-parsimony weight family. This is not merely a convenient parametrization: the Beta kernel is minimax-optimal under bounded misspecification, with worst-case efficiency loss only $O(\delta^2)$ (Theorem~\ref{thm:minimax}). Curriculum learning~\citep{bengio2009curriculum, kumar2010selfpaced} addresses related concerns but relies on fixed difficulty annotations or predetermined schedules; in distillation, difficulty depends on \emph{who} is solving each problem and \emph{when}---a problem intractable at epoch~1 may become productive by epoch~5.

We validate \textsc{Paced} across two settings: \textbf{Distillation} (Qwen3-8B\textsubscript{GRPO} $\to$ Qwen3-1.7B, forward KL) and \textbf{Self-distillation} (Qwen2.5-Math-7B-Instruct, reverse KL), with cross-family generalization to Llama-3.1-8B-Instruct. Our contributions:
\begin{enumerate}
  \item \textbf{A novel empirical finding with theoretical characterization.} We provide the first direct measurement of cross-problem gradient SNR in distillation, revealing its bell-shaped collapse as a function of student pass rate (Figure~\ref{fig:snr_empirical}). We then prove the Beta kernel is the leading-order optimal and minimax-robust weight family arising from this structure (Propositions~\ref{prop:boundary}--\ref{prop:representation}, Theorem~\ref{thm:minimax}).
  \item \textbf{Strong gains with near-zero forgetting.} \textsc{Paced} improves over unweighted distillation by up to $+3.9$ on AIME~2024 (forward KL) and up to $+8.2$ on AIME~2025 (reverse KL), while keeping MMLU forgetting $\leq 1.4\%$---matching the best-retention baseline while strictly outperforming it on reasoning. Gains generalize across Qwen3, Qwen2.5, and Llama-3 families.
  \item \textbf{A two-stage KL schedule.} Forward KL (mode-covering) followed by reverse KL (mode-seeking) pushes gains further to $+4.6/+4.9/+5.8$ over standard forward KL on MATH-500/AIME~2024/AIME~2025, confirming that the two KL directions are complementary.
\end{enumerate}
An overview appears in Figure~\ref{fig:overview}.

\begin{figure}[t]
  \centering
  \includegraphics[width=\textwidth]{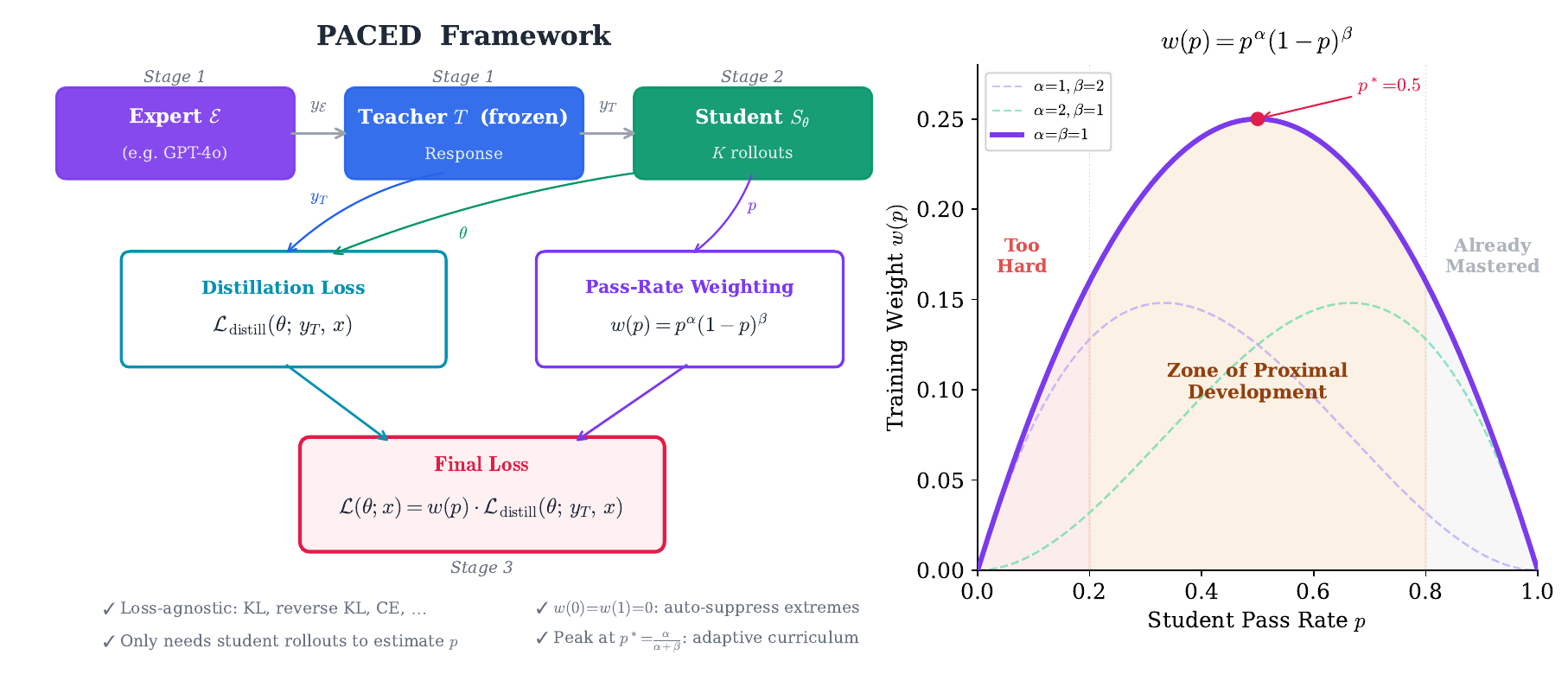}
  \caption{\textbf{Overview of \textsc{Paced}.} \emph{Left:} The pipeline---an expert provides reference solutions, and the student learns via a distillation loss weighted by pass rate. \emph{Right:} The Beta-kernel weighting $w(p)=p^\alpha(1-p)^\beta$ concentrates training on the zone of proximal development, suppressing trivial and intractable problems.}
  \label{fig:overview}
\end{figure}

\section{Related Work}
\label{sec:related_work}

\textbf{Knowledge Distillation.}
The idea of training a smaller model to mimic a larger one dates to~\citet{hinton2015distilling}, who showed that the ``soft'' distribution over classes carries richer information than hard labels alone. Since then, the field has explored sequence-level distillation~\citep{kim2016sequencekd}, reverse KL objectives~\citep{gu2024minillm, agarwal2024gkd}, distribution-aligned methods~\citep{yan2026dasd, boizard2024uld}, and regression-based approaches~\citep{ba2014deep, kim2021klmse, wang2020minilm}. A common thread runs through this work: all samples are treated alike. Our contribution is to break this symmetry, letting the student's own competence determine where training effort flows---regardless of the underlying loss function.

\textbf{Curriculum Learning.}
\citet{bengio2009curriculum} articulated the principle that models benefit from seeing easier examples first. Self-paced learning~\citep{kumar2010selfpaced} and automated curriculum design~\citep{graves2017automated} extended this intuition in various directions. However, existing approaches typically rely on fixed difficulty annotations or predetermined schedules; our Beta-kernel weight adapts automatically from the student's pass rate alone.

\textbf{Sample Reweighting and Catastrophic Forgetting.}
Importance sampling~\citep{katharopoulos2018not} and meta-learned weights~\citep{ren2018learning} require per-sample gradient computation; our Beta-kernel weight is a closed-form function of the pass rate alone.
In RL, ACE~\citep{xu2025ace} modulates per-rollout penalties within GRPO/DAPO; \textsc{Paced} operates at the problem level and the two are complementary.
Forgetting mitigation via parameter constraints~\citep{kirkpatrick2017ewc, lopezpaz2017gem, farajtabar2020ogd} is orthogonal to our approach, which prevents forgetting by filtering harmful training signals before they reach the optimizer.

\textbf{On-Policy Distillation and Self-Distillation.}
GKD~\citep{agarwal2024gkd} trains on student-generated samples; SDFT~\citep{shenfeld2026sdft} identifies student and teacher as the same model. Recent extensions address privileged traces~\citep{zhao2026sdr}, context-conditioned reverse KL~\citep{ye2026opcd}, and conciseness-conditioned self-distillation~\citep{sang2026opsdc}. Our pass-rate weighting is orthogonal: it determines \emph{which} problems to prioritize, regardless of loss or generation policy.

Table~\ref{tab:method_feature_matrix} summarizes the key design features that distinguish \textsc{Paced} from representative prior methods.\label{sec:method_positioning_summary}

\begin{table}[H]
\centering
\caption{\textbf{Method-feature comparison.} \checkmark\ = primary design characteristic.}
\vspace{-0.5em}
\footnotesize
\setlength{\tabcolsep}{4pt}
\begin{tabular}{@{}lccccc@{}}
\toprule
\textbf{Feature} & \textbf{Self-Dist.} & \textbf{AdaRFT} & \textbf{AdaKD} & \textbf{AKL} & \textbf{\textsc{Paced}} \\
\midrule
Adaptive weighting / curriculum &  & \checkmark & \checkmark & \checkmark & \checkmark \\
Student-side competence signal &  & \checkmark &  &  & \checkmark \\
Implicit forgetting reduction & \checkmark &  &  &  & \checkmark \\
Loss-agnostic &  &  & \checkmark &  & \checkmark \\
Theoretically grounded &  &  &  & \checkmark & \checkmark \\
\bottomrule
\end{tabular}
\label{tab:method_feature_matrix}
\end{table}

\section{Methodology}
\label{sec:method}
\textsc{Paced} rests on a single core idea: a weighting scheme that directs distillation toward the problems where it can do the most good (Section~\ref{sec:filtering}).

\subsection{Problem Setup}
\label{sec:setup}

We use two disjoint training splits: $\mathcal{D}^{\text{dist}}$ for distillation and $\mathcal{D}^{\text{self}}$ for self-distillation. Let $T$ denote the frozen teacher model and $S_\theta$ the student model. In distillation, $T$ is a GRPO-finetuned same-family model (Qwen3-8B trained with GRPO~\citep{shao2024deepseekmath} on DAPO-Math-17k, denoted Qwen3-8B\textsubscript{GRPO}) and $S_\theta$ is Qwen3-1.7B. In self-distillation, $T$ is a frozen copy of Qwen2.5-Math-7B-Instruct and $S_\theta$ is the trainable copy. In both settings, $T$ is fixed while $\theta$ is updated.

For each prompt $x$, distillation uses a teacher-side reference response $y_T$. To construct attainable targets, an external expert $\mathcal{E}$ (e.g., a frontier API or strong open-weight model) first produces a solution $y_{\mathcal{E}}$, and the frozen teacher $T$ then regenerates it in its own distributional voice, $y_T \sim P_T(y \mid x, y_{\mathcal{E}})$, producing a target naturally within the model family's expressive range. The student sees only $x$; the teacher sees $(x, y_{\mathcal{E}})$. This \emph{prompt asymmetry} converts black-box expert supervision into white-box, same-family distillation signals (full token-level logits rather than hard-label SFT); Appendix~\ref{app:prompts} provides the template.

To measure the student's current competence on $x$, we sample $K$ rollouts from the student and compute the \textbf{pass rate}:

\begin{equation}
  p(x; \theta) = \frac{1}{K} \sum_{k=1}^K \mathbbm{1}\left[\texttt{correct}(y_S^{(k)}, x)\right], \quad y_S^{(k)} \sim \pi_\theta(\cdot \mid x)
\end{equation}
The pass rate $p \in [0, 1]$ measures the student's current competence on problem $x$.

\subsection{Pass-Rate Weighting}
\label{sec:filtering}

\textbf{Motivation.} Given a reference target $y_T$ for each prompt, the core design question is how to weight problems according to the student's current competence. Not all training problems contribute equally. At one extreme ($p \approx 0$), the student cannot solve the problem at all; logit gradients are large but point in near-random directions across prompts, offering high variance and little useful signal. At the other extreme ($p \approx 1$), individual per-problem gradients need not vanish---for distribution-matching losses, the teacher's distribution may remain sharper than the student's---but the remaining corrections are problem-specific calibration refinements (e.g., sharpening predictions on algebraic tokens for one problem, adjusting geometric reasoning for another) that \emph{disperse} in parameter space. When averaged across mastered problems, these dispersed corrections largely cancel, yielding low cross-problem SNR even though each individual gradient carries signal. In practice a substantial fraction of problems falls into these low-SNR extremes---e.g., with Qwen3-1.7B on DAPO, roughly $49\%$ of problems have $p < 0.2$ or $p > 0.8$ (the exact proportion depends on the model and dataset). The richest---highest signal-to-noise ratio---gradient signal concentrates at \emph{intermediate} difficulty, where a coherent skill gap provides a shared gradient direction and each update advances the student's capability frontier. This raises a natural question: \emph{what is the principled weight function that exploits this structure?}

\textbf{Theoretical answer.} In distillation, the gradient signal-to-noise ratio (SNR) collapses at both boundaries via cross-problem gradient incoherence: at $p \to 0$ (diverse intractable problems) and $p \to 1$ (dispersed problem-specific refinements; Proposition~\ref{prop:boundary}). Under power-law regularity at the boundaries (Assumption~\ref{asm:passrate_structure}(b)), any such SNR profile decomposes as $p^{a'}(1-p)^{b'} \cdot e^{r(p)}$ with bounded remainder (Proposition~\ref{prop:representation}). The leading-order, maximum-parsimony weight family is therefore the Beta kernel:
\begin{equation}
  \boxed{w(p) = p^{\alpha}(1-p)^{\beta}}
  \label{eq:beta_kernel}
\end{equation}
with peak at $p^* = \alpha/(\alpha + \beta)$. The default choice $\alpha = \beta = 1$ gives $w(p) = p(1-p)$, which is symmetric around $p^* = 0.5$, zero at the boundaries, and equals the inverse Bernoulli Fisher information (Remark~\ref{rem:fisher}). Asymmetric choices ($\alpha \neq \beta$) shift the peak to prioritize harder or easier problems. This form is minimax-robust: even when the true SNR profile deviates from the Beta model by a multiplicative factor $e^{\pm\delta}$, the worst-case efficiency loss is only $O(\delta^2)$ (Theorem~\ref{thm:minimax}). See Appendix~\ref{app:proof_minimax} for the full derivation.

\subsection{Overall Algorithm}
\label{sec:overall}

Putting the pieces together: each problem's contribution to the loss is scaled by how informative it is for the student right now:
\begin{equation}
  \mathcal{L}(\theta; x) = w(p) \cdot \mathcal{L}_{\text{distill}}(\theta; x)
  \label{eq:per_sample_loss}
\end{equation}
where $p = p(x; \theta)$, $w(p) = p(1-p)$ by default, and $\mathcal{L}_{\text{distill}}$ is chosen by training setting. To keep the instantiation clean, we bind one KL direction to each setting: distillation (Qwen3-8B\textsubscript{GRPO} $\to$ Qwen3-1.7B) uses forward KL, and self-distillation (Qwen2.5-Math-7B-Instruct) uses reverse KL. This pairing reflects their roles: forward KL favors broad teacher-mode coverage when the teacher is stronger, while reverse KL favors compact, high-confidence modes when teacher and student are near-policy. Concretely, we use:
\begin{itemize}
  \item \textbf{Distillation track (Qwen3): Forward KL} along the teacher sequence $y_T$: $\sum_t D_{KL}\!\bigl(p_T(\cdot \mid y_{T,<t})\,\|\,p_S(\cdot \mid y_{T,<t})\bigr)$.
  \item \textbf{Self-distillation track (Qwen2.5): Reverse KL} along a student sequence $y_S \sim \pi_\theta(\cdot \mid x)$: $\sum_t D_{KL}\!\bigl(p_S(\cdot \mid y_{S,<t})\,\|\,p_T(\cdot \mid y_{S,<t})\bigr)$.
\end{itemize}
\begin{equation}
  \mathcal{L}(\theta) = \frac{1}{N} \sum_{i=1}^{N} \mathcal{L}(\theta; x_i)
  \label{eq:total_loss}
\end{equation}

\begin{algorithm}[H]
\caption{\textsc{Paced}: Competence-Aware Distillation with Pass-Rate Weighting}
\label{alg:main}
\begin{algorithmic}[1]
\REQUIRE Prompt dataset $\mathcal{D}$, expert $\mathcal{E}$, frozen teacher $T$, student $S_\theta$, distillation loss $\mathcal{L}_{\text{distill}}$ (forward KL or reverse KL), weight exponents $(\alpha, \beta)$ (default $\alpha{=}\beta{=}1$), rollouts $K$
\STATE \textbf{// Stage 1: Teacher-side target preparation (forward KL only)}
\IF{$\mathcal{L}_{\text{distill}}$ is forward KL}
\FOR{each prompt $x \in \mathcal{D}$}
\STATE $y_{\mathcal{E}} \leftarrow \mathcal{E}(x)$ \hfill\COMMENT{Expert rollout (e.g., GPT-family API solution)}
 \STATE $y_{T} \leftarrow T(\cdot \mid x, y_{\mathcal{E}})$ \hfill\COMMENT{Teacher regeneration conditioned on expert solution}
\ENDFOR
\ENDIF
\STATE \textbf{// Stage 2: One-shot pass-rate estimation (paper setting)}
\FOR{each prompt $x_i \in \mathcal{D}$}
\STATE Sample $\{y_{S,i}^{(k)}\}_{k=1}^K \sim \pi_\theta(\cdot \mid x_i)$
\STATE $p_i \leftarrow \frac{1}{K}\sum_k \mathbbm{1}[\texttt{correct}(y_{S,i}^{(k)}, x_i)]$
\STATE $w_i \leftarrow p_i^{\alpha}(1-p_i)^{\beta}$ \hfill\COMMENT{Default: $w_i = p_i(1-p_i)$}
\ENDFOR
\STATE $\tilde{w}_i \leftarrow w_i \,/\, \bar{w}$ for all $i$ \hfill\COMMENT{Normalize to unit mean and keep fixed during training}
\STATE \textbf{// Stage 3: Weighted Distillation}
\FOR{each training iteration}
 \FOR{each prompt $x_i \in \mathcal{D}$}
 \IF{$\mathcal{L}_{\text{distill}}$ is forward KL}
 \STATE $\mathcal{L}(x_i) \leftarrow \tilde{w}_i \cdot \mathcal{L}_{\text{distill}}(\theta; y_{T,i}, x_i)$ \hfill\COMMENT{Teacher-forced distillation}
 \ELSE
 \STATE Sample $y_{S,i} \sim \pi_\theta(\cdot \mid x_i)$
 \STATE $\mathcal{L}(x_i) \leftarrow \tilde{w}_i \cdot \mathcal{L}_{\text{distill}}(\theta; y_{S,i}, x_i)$ \hfill\COMMENT{Reverse-KL self-distillation on student rollouts}
 \ENDIF
 \ENDFOR
 \STATE Update $\theta$ via gradient descent on $\frac{1}{N}\sum_i \mathcal{L}(x_i)$
\ENDFOR
\STATE \textbf{// Optional extension:} periodically recompute $\{p_i,\tilde{w}_i\}$ every $T_0$ steps
\end{algorithmic}
\end{algorithm}

\textbf{Iterative Refinement.} In most experiments we estimate pass rates once before optimization and keep the resulting weights fixed (single-pass weighting); this suffices because the Beta kernel is minimax-robust to stale pass rates (Theorem~\ref{thm:minimax}). When compute allows, pass rates can be recomputed periodically or at stage boundaries to track evolving competence; Appendix~\ref{app:ablation_recompute} ablates both options.

\subsection{Theoretical Guarantees}
\label{sec:theory}
We provide theoretical justification for Beta-kernel weighting. The core intuition is simple: at both pass-rate extremes, \emph{cross-problem gradient incoherence} drives the SNR to zero, but for structurally distinct reasons. When $p \to 0$, the student is far from solving the problem, so gradients from diverse intractable prompts are poorly aligned, yielding low SNR. When $p \to 1$, per-problem gradients from distribution-matching losses need not vanish (the teacher may remain sharper), but the remaining corrections are problem-specific refinements that disperse in parameter space---unlike intermediate pass rates where a coherent skill gap provides a shared gradient direction. This suggests that useful learning signal should concentrate in the interior of the pass-rate range, and that a principled weight should be near zero at both boundaries while remaining flexible about where to peak. The results below formalize exactly this picture, showing that under mild regularity the Beta family is the leading-order and robust choice. Full proofs, assumptions, and regime distinctions appear in Appendix~\ref{app:proofs}. The assumptions rest on standard regularity conditions from stochastic optimization~\citep{ghadimi2013sgd, bottou2018optimization}, well-documented properties of gradient statistics in knowledge distillation~\citep{tang2020understanding, sankararaman2020gradient, agarwal2022vog}, and our own empirical SNR measurements. In particular, Appendix~\ref{app:snr} and Figure~\ref{fig:snr_empirical} directly visualize the predicted bell-shaped SNR profile and show that it closely tracks $\sqrt{p(1-p)}$.
\begin{resultbox}[Key Results at a Glance]
\begin{itemize}
\item \textbf{Structural characterization:} Distillation gradient SNR collapses at both boundaries via cross-problem gradient incoherence (Prop~\ref{prop:boundary}); with power-law regularity, any such profile decomposes as $p^{a'}(1{-}p)^{b'} \cdot e^{r(p)}$ with bounded $r$ (Prop~\ref{prop:representation}), yielding the Beta kernel as the leading-order weight family
\item \textbf{Minimax robustness (Main Theorem):} Under bounded misspecification $|r(p)| \leq \delta$, the Beta kernel is minimax-optimal for the low-SNR approximation, with worst-case efficiency $\text{sech}^2(\delta) \geq 1 - \delta^2$, both pointwise and in aggregate (Thm~\ref{thm:minimax})
\item \textbf{Batch-level variance reduction:} $R < 1$ when $-\text{Cov}(\tilde{w}^2, s^2) > \text{Var}(\tilde{w})\,\mathbb{E}[s^2]$ (Prop~\ref{prop:var_reduction})
\item \textbf{Exponent selection:} $(\alpha^*{+}1)/(\alpha^*{+}\beta^*{+}2) = \bar{p}_{\mathcal{Z}}$, $\alpha^*{+}\beta^* = \bar{p}_{\mathcal{Z}}(1{-}\bar{p}_{\mathcal{Z}})/\text{Var}_{\mathcal{Z}}(p) - 3$ (Prop~\ref{prop:exponent_selection})
\end{itemize}
\end{resultbox}

Full proofs appear in Appendix~\ref{app:proofs}, which opens with the empirical SNR measurement (Section~\ref{app:snr}) that motivates the boundary-collapse assumption before developing the formal theory. We briefly unpack the intuition behind Results 3 and 4, which are less immediate than the first two.

\textbf{Intuition for Result 3.}
Non-uniform weighting faces a tug of war: downweighting reduces the effective batch size (increasing variance), but if the downweighted samples are precisely the noisiest ones, the net effect is variance \emph{reduction}. Because gradient noise runs hottest at extreme pass rates~\citep{agarwal2022vog}, the Beta kernel's near-zero boundary weights suppress exactly those high-variance samples, making the coupling term sufficiently negative to win the trade-off ($R<1$). Appendix~\ref{app:proof_convergence} identifies concrete parameter regimes.

\textbf{Intuition for Result 4.}
The moment-matching formula requires only the ZPD pass-rate mean and variance---no gradient computation. When informative problems cluster tightly, it prescribes a peaked kernel; when they spread broadly, a flatter one. See Appendix~\ref{app:proof_exponent} for the derivation.

\textbf{Forgetting reduction.}
By suppressing gradient updates from boundary-pass-rate samples, Beta-kernel weighting substantially reduces catastrophic forgetting; see Tables~\ref{tab:forgetting_distill} and \ref{tab:forgetting_self}.

\section{Experiments}
\label{sec:experiments}

\subsection{Experimental Setup}

\begin{itemize}
    \item \textbf{Training data:} DAPO-Math-17k~\citep{yu2025dapo}. We use two disjoint prompt splits, one for distillation and one for self-distillation.
    \item \textbf{External Expert:} gpt-oss-120b~\citep{openai2025gptoss} for initial solution generation.
    \item \textbf{Teacher/Student Models (split by setting):}
    \begin{itemize}
      \item \textbf{Distillation setting:} Qwen3-1.7B~\citep{qwen2025qwen3} as student, frozen Qwen3-8B\textsubscript{GRPO} (Qwen3-8B finetuned with GRPO on DAPO-Math-17k) as teacher, and \textbf{forward KL} as base loss.
      \item \textbf{Self-distillation setting:} Qwen2.5-Math-7B-Instruct~\citep{yang2024qwen25math} with a frozen self-teacher, and \textbf{reverse KL} as base loss.
    \end{itemize}
    In both settings, the teacher is frozen throughout training; forward-KL targets are teacher regenerations conditioned on expert solutions, whereas reverse KL is computed on student rollouts.
    \item \textbf{Evaluation:}
    \begin{itemize}
      \item \emph{Plasticity} (new skill acquisition): 8-sample mean accuracy on MATH-500~\citep{hendrycks2021math}, AIME 2024, and AIME 2025 (out-of-distribution generalization). For each problem, we sample 8 responses (temperature 0.6, top-$p$ 0.95), compute the fraction of correct samples, and then average this fraction over problems. The $\pm$ intervals in Tables~\ref{tab:main_distill}--\ref{tab:main_self} denote the sampling standard deviation of this mean across problems.
      \item \emph{Stability} (retention of prior knowledge): a fixed random subsample of 2{,}000 questions from MMLU~\citep{hendrycks2021mmlu}. Correctness for pass-rate estimation and the reasoning-benchmark evaluations is determined by normalized final-answer matching, whereas MMLU is evaluated with \texttt{lm-evaluation-harness}~\citep{eval-harness} using 5-shot prompting; details are in Appendix~\ref{app:hyperparameters}.
    \end{itemize}
    \item \textbf{Rollouts:} $K = 8$ rollouts per problem for pass-rate estimation.
    \item \textbf{Pass-rate weight:} Default $w(p)=p(1-p)$ (i.e., $\alpha=\beta=1$). Unless otherwise noted, pass rates are estimated once before optimization; the resulting weights are normalized to unit mean (i.e., $\tilde{w}_i = w_i / \bar{w}$) and kept fixed during training. Appendix~\ref{app:ablation_recompute} ablates periodic recomputation.
    \item \textbf{Training:} AdamW for 2 epochs with global batch size 32 and constant learning rate $1 \times 10^{-7}$.
    \item \textbf{Baselines (setting-specific):}
    \begin{itemize}
      \item \textbf{Distillation/Qwen3:} Forward KL (unweighted), Hard Filter Forward KL, AKL, and \textsc{Paced} Forward KL.
      \item \textbf{Self-distillation/Qwen2.5:} Reverse KL (unweighted), Hard Filter Reverse KL, AKL, and \textsc{Paced} Reverse KL.
      \item \textbf{Hard Filter:} binary problem selection retaining problems with $0.2 \le p \le 0.8$ (i.e., 2 through 6 of 8 rollouts correct); other hyperparameters match the unweighted baseline.
      \item \textbf{AKL~\citep{wu2025akl}:} a token-level adaptive KL baseline that dynamically adjusts the per-token KL coefficient based on teacher--student logit discrepancy. Unlike \textsc{Paced}, it requires no rollout or pass-rate estimation; all other optimization hyperparameters are matched to the corresponding unweighted baseline.
    \end{itemize}
\end{itemize}

\subsection{Main Results (Plasticity-Stability Trade-off)}

\begin{table}[h]
\caption{Distillation track (Qwen3-8B\textsubscript{GRPO} $\rightarrow$ Qwen3-1.7B, forward KL family): reasoning performance (8-sample mean accuracy). $\uparrow$ = higher is better.}
\centering
\begin{tabular}{lccc}
\toprule
\textbf{Method} & \textbf{MATH-500} ($\uparrow$) & \textbf{AIME 2024} ($\uparrow$) & \textbf{AIME 2025} ($\uparrow$) \\
\midrule
Base & \mstd{69.4}{0.4} & \mstd{11.5}{0.9} & \mstd{7.6}{0.7} \\
Forward KL (unweighted) & \mstd{76.8}{0.3} & \mstd{21.2}{1.3} & \mstd{17.0}{0.9} \\
Hard Filter Forward KL & \mstd{78.5}{0.6} & \mstd{23.7}{0.9} & \mstd{18.8}{0.6} \\
AKL & \mstd{77.6}{0.4} & \mstd{23.9}{1.2} & \mstd{19.1}{0.8} \\
\textbf{\textsc{Paced} Forward KL} & \textbf{\mstd{79.4}{0.5}} & \textbf{\mstd{25.1}{1.0}} & \textbf{\mstd{20.6}{0.7}} \\
\bottomrule
\end{tabular}
\label{tab:main_distill}
\end{table}

\begin{table}[h]
\caption{Self-distillation track (Qwen2.5-Math-7B-Instruct, reverse KL family): reasoning performance (8-sample mean accuracy).}
\centering
\begin{tabular}{lccc}
\toprule
\textbf{Method} & \textbf{MATH-500} ($\uparrow$) & \textbf{AIME 2024} ($\uparrow$) & \textbf{AIME 2025} ($\uparrow$) \\
\midrule
Base & \mstd{83.9}{0.6} & \mstd{19.6}{1.0} & \mstd{11.5}{0.7} \\
Reverse KL (unweighted) & \mstd{88.9}{0.5} & \mstd{25.3}{1.2} & \mstd{16.9}{1.1} \\
Hard Filter Reverse KL & \mstd{92.0}{0.5} & \mstd{28.9}{1.3} & \mstd{22.0}{0.9} \\
AKL & \mstd{91.4}{0.5} & \mstd{28.2}{0.8} & \mstd{21.5}{0.6} \\
\textbf{\textsc{Paced} Reverse KL} & \textbf{\mstd{93.7}{0.6}} & \textbf{\mstd{31.6}{1.1}} & \textbf{\mstd{25.1}{0.7}} \\
\bottomrule
\end{tabular}
\label{tab:main_self}
\end{table}

\begin{table}[h]
\caption{Retention in distillation track (Qwen3 forward KL family): MMLU and forgetting ($\Delta$ from base).}
\centering
\begin{tabular}{lccc}
\toprule
\textbf{Method} & \textbf{MMLU} ($\uparrow$) & \textbf{Forgetting} ($\downarrow$) & \textbf{Weighting} \\
\midrule
Base & \mstd{51.2}{0.2} & -- & -- \\
Forward KL (unweighted) & \mstd{48.3}{0.3} & 2.9\% & None \\
Hard Filter Forward KL & \mstd{49.8}{0.5} & 1.4\% & Hard \\
AKL & \mstd{49.5}{0.4} & 1.7\% & Token-level \\
\textbf{\textsc{Paced} Forward KL} & \textbf{\mstd{49.8}{0.4}} & \textbf{1.4\%} & Beta \\
\bottomrule
\end{tabular}
\label{tab:forgetting_distill}
\end{table}

\begin{table}[h]
\caption{Retention in self-distillation track (Qwen2.5 reverse KL family): MMLU and forgetting ($\Delta$ from base).}
\centering
\begin{tabular}{lccc}
\toprule
\textbf{Method} & \textbf{MMLU} ($\uparrow$) & \textbf{Forgetting} ($\downarrow$) & \textbf{Weighting} \\
\midrule
Base & \mstd{70.6}{0.5} & -- & -- \\
Reverse KL (unweighted) & \mstd{68.4}{0.4} & 2.2\% & None \\
\textbf{Hard Filter Reverse KL} & \textbf{\mstd{70.1}{0.4}} & \textbf{0.5\%} & Hard \\
AKL & \mstd{69.8}{0.3} & 0.8\% & Token-level \\
\textsc{Paced} Reverse KL & \mstd{70.0}{0.3} & 0.6\% & Beta \\
\bottomrule
\end{tabular}
\label{tab:forgetting_self}
\end{table}

\textbf{Reasoning (Tables~\ref{tab:main_distill} and \ref{tab:main_self}).}
The pattern is consistent across both tracks. In distillation (Qwen3, forward KL), \textsc{Paced} improves over unweighted forward KL by $+2.6/+3.9/+3.6$ on MATH-500/AIME~2024/AIME~2025. In self-distillation (Qwen2.5, reverse KL), \textsc{Paced} improves over the unweighted baseline by $+4.8/+6.3/+8.2$ on the same three benchmarks, or $+9.8/+12.0/+13.6$ relative to the base model.

\textbf{AKL baseline comparison.}
AKL~\citep{wu2025akl} adapts the KL coefficient \emph{per token} based on logit discrepancy, whereas \textsc{Paced} modulates \emph{per problem} via pass rate.
\textsc{Paced} consistently outperforms AKL on all reasoning benchmarks in both tracks. The gap reflects a structural limitation of token-level schemes: a problem with $p \approx 0$ produces unreliable gradients at \emph{every} token, yet AKL still trains on it because no individual token triggers a large enough logit gap to be downweighted. Problem-level weighting can suppress such intractable (or fully mastered) problems entirely.
The two approaches are orthogonal and could be combined; see Appendix~\ref{app:akl_analysis}.

\textbf{Cross-family generalization.}
To verify that the gains are not specific to the Qwen family, we replicate the distillation experiment on Llama-3.1-8B-Instruct~\citep{grattafiori2024llama3}. The same pattern holds; see Appendix~\ref{app:llama} and Table~\ref{tab:llama_self} for full results.

\textbf{Stability (Tables~\ref{tab:forgetting_distill} and \ref{tab:forgetting_self}).}
In distillation, \textsc{Paced} forward KL reduces forgetting from $2.9$ to $1.4$ percentage points---matching Hard Filter exactly on retention while outperforming it by $+0.9/+1.4/+1.8$ on MATH-500/AIME~2024/AIME~2025. This shows that smooth Beta-kernel weighting recovers the full stability benefit of binary thresholding without sacrificing any reasoning gains. In self-distillation, reverse-KL-based methods already forget less; \textsc{Paced} preserves the strongest reasoning gains, while Hard Filter is slightly better on retention alone ($0.5\%$ vs. $0.6\%$ forgetting). These stability gains align with the competence-distribution picture: roughly $17\%$ of problems have $p<0.2$ and $32\%$ have $p>0.8$ at initialization, so the $p(1{-}p)$ kernel suppresses both tails and concentrates weight on the informative interior. As training proceeds, problems migrate through the ZPD into the mastered regime (Appendix~\ref{app:curriculum}), and empirical gradient SNR exhibits the predicted bell-shaped profile (Appendix~\ref{app:snr}).

\subsection{Two-Stage KL Schedule}
\label{sec:twostage}
With the single-loss picture established, we ask whether staging the KL direction yields further gains.
Forward KL is mode-covering---it spreads the student toward the teacher's stronger reasoning modes; reverse KL is mode-seeking---it sharpens the student onto high-confidence modes.
A natural hypothesis is that mode-coverage should come first, followed by consolidation.
In a matched two-stage Qwen3 experiment (one midpoint pass-rate recomputation, 50/50 budget split), the KL $\to$ RevKL schedule yields $\mathbf{81.4/26.1/22.8}$ on MATH-500/AIME~2024/AIME~2025---improving over single-loss Paced KL by $+1.7/+0.5/+1.7$---while the reversed order underperforms both single-loss references, confirming the mode-coverage-then-consolidation interpretation.
Full order comparison and stage-budget ablation appear in Appendix~\ref{app:twostage_budget} (Table~\ref{tab:twostage_bridge}).

\textbf{Design sensitivity.} Additional ablations in Appendix~\ref{app:ablation_exponents}, Appendix~\ref{app:ablation_K}, and Appendix~\ref{app:ablation_recompute} show that the default symmetric kernel $w(p)=p(1-p)$ offers the best plasticity--stability trade-off in our setting, while smaller rollout budgets and single-pass weighting already retain most of the gains. We therefore keep $(\alpha,\beta)=(1,1)$ and single-pass $K{=}8$ as the main-text defaults.

\section{Discussion: Limitations and Future Work}
\label{sec:discussion}

Several limitations deserve candid acknowledgment.
\textbf{Rollout overhead.} Pass-rate estimation requires $K$ additional student rollouts per problem. With the default $K{=}8$, this adds roughly one inference pass over the training set before optimization begins. The ablations in Appendix~\ref{app:ablation_K} show that $K{=}4$ already captures most of the benefit, halving this cost.
\textbf{Pass-rate granularity and hard-zero boundary.} With $K{=}8$ rollouts, the estimated pass rate takes only nine discrete values $\{0, 1/8, \dots, 1\}$. The Beta kernel's smooth shape mitigates discretization (nearby values receive similar weights), but the granularity may limit the method's ability to distinguish fine competence differences. A related consequence is that a problem with true $p \approx 0.05$ may be estimated as $\hat{p}=0$, receiving exactly zero weight and losing any weak learning signal it might carry. This hard-zero boundary effect is inherent to the Beta kernel's $w(0)=0$ property and becomes more pronounced at small $K$. Larger $K$ improves resolution at the cost of additional rollouts; Table~\ref{tab:ablation_K} shows diminishing returns beyond $K{=}8$.
\textbf{Dependence on explicit correctness signal.} Pass-rate estimation requires a binary correctness verdict for each rollout. This is straightforward for domains with verifiable answers (mathematics, code execution), but does not directly apply to open-ended tasks---such as creative writing, summarization, or general instruction following---where no unambiguous ground-truth reward exists. Extending \textsc{Paced} to such settings would require a proxy competence signal (e.g., reward-model scores or LLM-as-judge evaluations), which introduces its own noise and calibration challenges.
\textbf{Future work.} Natural extensions include continuous KL interpolation between forward and reverse objectives, cross-architecture distillation (where the gradient-dispersion assumption at mastery may behave differently), multi-teacher ensembles, and integration with token-level adaptive methods such as AKL.
\section{Conclusion}

\textsc{Paced} operationalizes a simple teaching principle---focus where the student is learning---for LLM distillation. Beta-kernel pass-rate weighting concentrates gradient budget on the frontier of competence, delivering $+2.6/+3.9/+3.6$ over standard forward KL on MATH-500/AIME~2024/AIME~2025 in distillation, and $+4.8/+6.3/+8.2$ over standard reverse KL in self-distillation, while keeping MMLU forgetting at $1.4\%$ and $0.6\%$, respectively. A two-stage forward-KL-then-reverse-KL schedule pushes gains further to $+4.6/+4.9/+5.8$ over standard forward KL, and the pattern generalizes across model families (Qwen3, Qwen2.5, Llama-3.1).

A practical strength of the framework is that single-pass pass-rate estimation already suffices: the minimax robustness guarantee (Theorem~\ref{thm:minimax}) ensures that even when weights become stale as the student improves, worst-case efficiency loss is only $O(\delta^2)$. The default kernel $w(p)=p(1-p)$ requires no exponent tuning and is robust across all model families and training settings tested. This makes \textsc{Paced} lightweight to deploy---one rollout pass before training, zero architectural changes, and zero hyperparameters.

\renewcommand{\bibfont}{\small}
\bibliographystyle{plainnat}
\bibliography{references}

\newpage
\appendix
\onecolumn
The appendix is organized into four parts: theory and proofs, prompts and implementation details, additional experiments, and additional interpretations.

\section{Theory and Proofs}
\label{app:proofs}

\paragraph{Proof roadmap.}
The main-text results map to the appendix as follows:
\begin{itemize}
  \item \textbf{Result 1 (Structural characterization):} Section~\ref{app:boundary_representation}, Propositions~\ref{prop:boundary} and \ref{prop:representation}.
  \item \textbf{Result 2 (Minimax robustness):} Section~\ref{app:proof_minimax}, Theorem~\ref{thm:minimax}.
  \item \textbf{Result 3 (Batch-level variance reduction):} Section~\ref{app:proof_convergence}, Proposition~\ref{prop:var_reduction} and the convergence analysis.
  \item \textbf{Result 4 (Exponent selection):} Section~\ref{app:proof_exponent}, Proposition~\ref{prop:exponent_selection}.
  \item \textbf{Additional intuition:} Corollary~\ref{prop:zpd} (in Section~\ref{app:boundary_representation}) explains why learning signal peaks at intermediate pass rates.
  \item \textbf{Complementary derivation:} Section~\ref{app:proof_beta}, Theorem~\ref{thm:beta_optimal}, derives the Beta family from per-problem descent maximization.
\end{itemize}
We begin with empirical evidence for the SNR boundary-collapse assumption, followed by the boundary conditions and representation theorem, the non-monotonic learning-signal corollary (Corollary~\ref{prop:zpd}), the complementary descent derivation, minimax robustness, variance analysis, and exponent selection.

\subsection{Empirical Motivation: Cross-Problem Gradient SNR}
\label{app:snr}

Before developing the formal theory, we present the empirical observation that motivates the entire framework. Figure~\ref{fig:snr_empirical} directly measures the cross-problem gradient SNR as a function of student pass rate.

For each problem $i$, the teacher generates one reference solution $y_{T,i}$, and we compute the forward-KL distillation gradient $g_i = \nabla_\theta \sum_t D_{KL}(p_T(\cdot \mid y_{T,i,<t}) \| p_S(\cdot \mid y_{T,i,<t}))$ with respect to the \texttt{lm\_head} parameters. Since the teacher sequence is fixed, each problem yields a single deterministic gradient vector. Student rollouts ($K{=}10$ per problem) are used only to estimate the pass rate $\hat{p}_i$.

We then group problems into equal-width pass-rate bins $\mathcal{B}_j = \{i : \hat{p}_i \in \text{bin}_j\}$ and compute the \textbf{cross-problem SNR} within each bin:
\begin{equation}
  \widehat{\text{SNR}}_j \;=\; \frac{\|\bar{g}_j\|_2}{\sqrt{\frac{1}{|\mathcal{B}_j|}\sum_{i \in \mathcal{B}_j}\|g_i - \bar{g}_j\|_2^2}}, \qquad \bar{g}_j = \tfrac{1}{|\mathcal{B}_j|}\textstyle\sum_{i \in \mathcal{B}_j} g_i,
\end{equation}
where the numerator measures the magnitude of the mean gradient across problems in the bin (signal) and the denominator measures the spread of individual problem gradients around this mean (cross-problem noise). The bin-level SNR values are rescaled to $[0,1]$ by dividing by the largest bin value.

To compare the empirical bars to the theory at the \emph{bin} level, let
\begin{equation}
  \widehat{\text{SNR}}_j^{\,\mathrm{norm}} = \frac{\widehat{\text{SNR}}_j}{\max_k \widehat{\text{SNR}}_k}, \qquad
  \text{SNR}_{\mathrm{th},j}^{\,\mathrm{norm}} = \frac{\sqrt{\bar{p}_j(1-\bar{p}_j)}}{\max_k \sqrt{\bar{p}_k(1-\bar{p}_k)}}.
  \label{eq:snr_normalization}
\end{equation}
Under the leading-order symmetric model $\text{SNR}^2(p) \propto p(1-p)$ used throughout the paper, the theoretical prediction for each bin depends only on its mean pass rate. With 10 equal-width bins the bin nearest $p=0.5$ has the largest $\bar{p}_k(1-\bar{p}_k)$, so its normalized height is $1.0$; other bins scale down accordingly. For example, bins with $\bar{p}_j \approx 0.1$ or $0.9$ have theoretical height $\approx 0.60$, and bins with $\bar{p}_j \approx 0.2$ or $0.8$ have $\approx 0.80$. The bell-shaped profile is clearly visible in both panels: cross-problem SNR peaks at intermediate pass rates---where a coherent skill gap provides a shared gradient direction---and is substantially lower at both boundaries, closely tracking the theoretical prediction. At $p \approx 1$, the low SNR directly reflects gradient dispersion: individual per-problem gradients are non-negligible, but they point in diverse directions across mastered problems, canceling when averaged. Since the theoretically optimal weight is proportional to $\text{SNR}^2$, this confirms that the default $p^\alpha(1{-}p)^\beta$ kernel is well-matched to the empirical gradient structure. The formal theory in the following sections takes this boundary-collapse structure as a modeling assumption and derives the Beta kernel as the leading-order, minimax-robust weight family.

\begin{figure}[h]
  \centering
  \begin{minipage}[t]{0.48\textwidth}
    \centering
    \includegraphics[width=\textwidth]{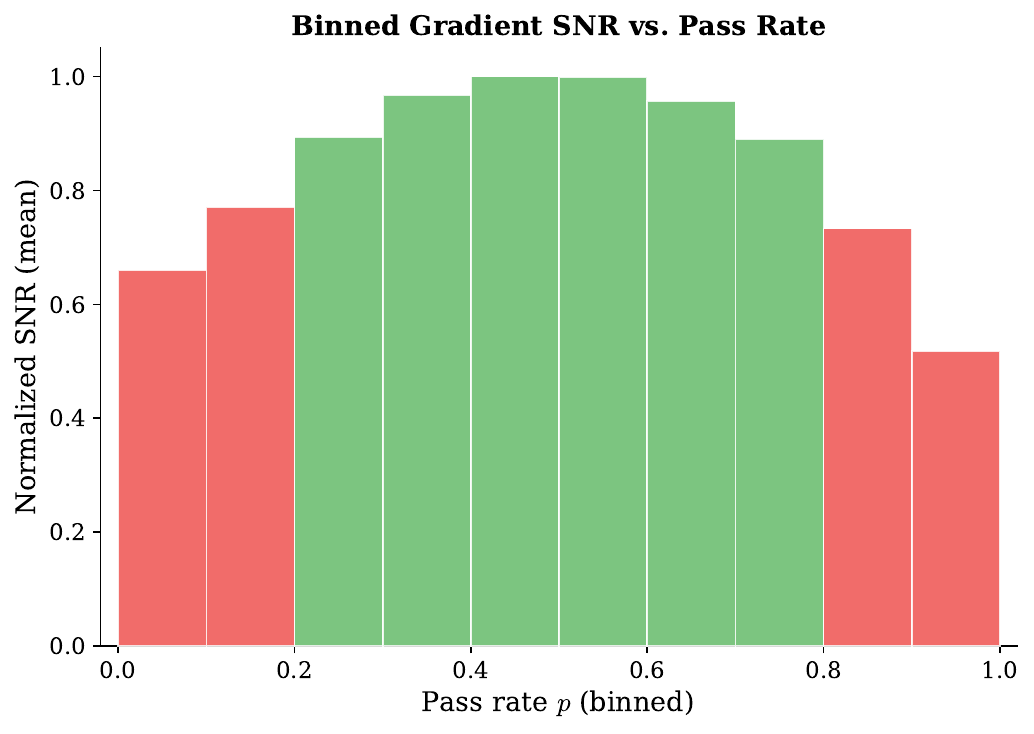}
    \centerline{\small (a) Step 1}
  \end{minipage}
  \hfill
  \begin{minipage}[t]{0.48\textwidth}
    \centering
    \includegraphics[width=\textwidth]{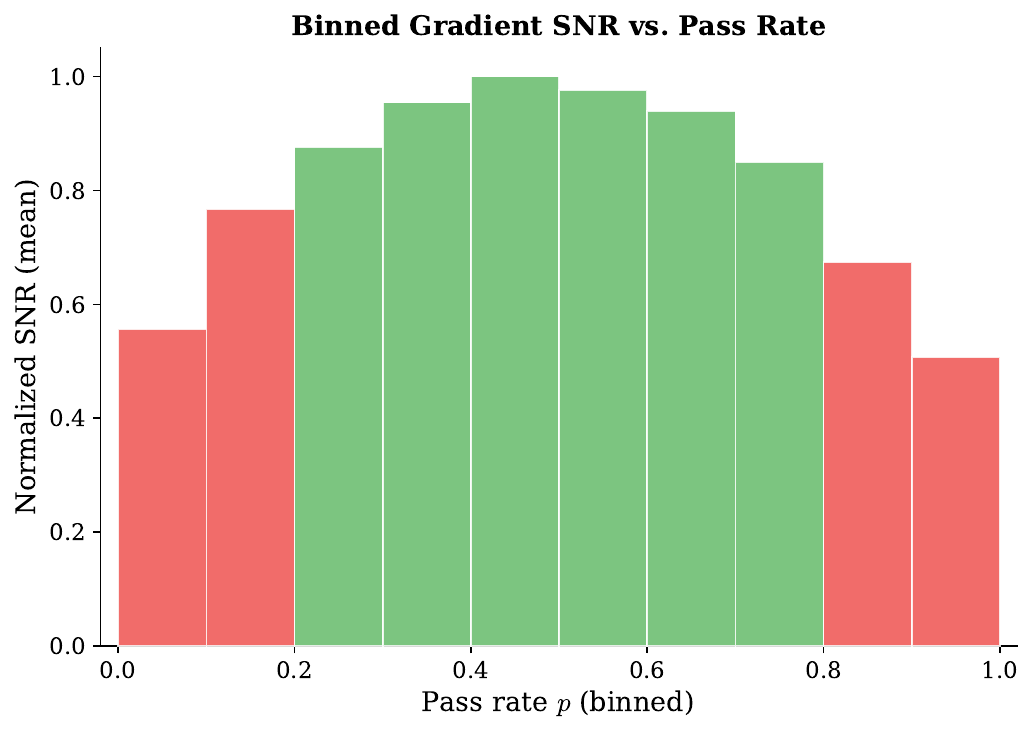}
    \centerline{\small (b) Step 20}
  \end{minipage}
  \caption{\textbf{Cross-problem gradient SNR vs.\ student pass rate at two training stages} (Qwen3-1.7B, forward KL). Left: Step~1. Right: Step~20. Each problem contributes one gradient (from its fixed teacher reference); $K{=}10$ student rollouts are used only for pass-rate estimation. Problems are grouped into equal-width pass-rate bins. Both empirical and theoretical values are normalized by dividing by the respective bin maximum (Eq.~\ref{eq:snr_normalization}), so the tallest bar is $1.0$ in both cases. The bell-shaped profile persists across training stages, confirming that the boundary collapse of cross-problem gradient coherence is a structural property of the distillation landscape, not a transient artifact of initialization.}
  \label{fig:snr_empirical}
\end{figure}

\subsection{Notation and Assumptions}
\label{app:notation_assumptions}

\textbf{Notation.} Throughout the appendix, $p \in [0,1]$ denotes the student pass rate for a problem, and $w(p)\ge 0$ denotes its pass-rate weight (typically a Beta kernel $w(p)=p^\alpha(1-p)^\beta$). We collect the shared assumptions here to avoid forward references in later proofs.

\textbf{Symbol guide.} Three distinct pairs of exponents appear in the analysis and should not be confused:
\begin{itemize}
  \item $(a_s, b_s)$: \emph{signal exponents}---govern how the expected gradient norm $\|\mathbb{E}[g(p)]\|$ scales with $p$ near the boundaries (Assumption~\ref{asm:passrate_structure}(a)).
  \item $(a', b')$: \emph{SNR boundary exponents}---govern the power-law decay of $\text{SNR}^2(p)$ at $p\to 0^+$ and $p\to 1^-$ (Assumption~\ref{asm:passrate_structure}(b)); these determine the shape of the theoretically optimal weight.
  \item $(\alpha, \beta)$: \emph{Beta kernel exponents}---the practitioner-facing hyperparameters in $w(p)=p^\alpha(1-p)^\beta$ (default $\alpha{=}\beta{=}1$).
\end{itemize}

\begin{assumption}[\textbf{Regularity Conditions}]
\label{asm:regularity}
(i)~The total loss $\mathcal{L}(\theta)$ is $L$-smooth; (ii)~per-sample gradients are unbiased; (iii)~per-sample gradient variance is bounded by $\sigma_0^2$. These are standard conditions in stochastic optimization~\citep{ghadimi2013sgd, bottou2018optimization}.
\end{assumption}

\begin{assumption}[\textbf{Bounded Logits and Jacobian}]
\label{asm:bounded_logits}
For all training steps and vocabulary dimensions $v$, the student and teacher logits are bounded as $|l_{S,v}|, |l_{T,v}| \leq B$, and the Jacobian of the student logits with respect to parameters satisfies $\|J_\theta\|_{\text{op}} = \|\partial l_S / \partial \theta\|_{\text{op}} \leq C_J$ for some constants $B, C_J > 0$. In practice, finite-precision arithmetic and weight decay ensure bounded logits, and the Jacobian bound holds for networks with bounded weights and Lipschitz activations.
\end{assumption}

\begin{assumption}[\textbf{Pass-Rate-Dependent Gradient Structure}]
\label{asm:passrate_structure}
The gradient statistics depend on pass rate $p$ through:
\begin{enumerate}
 \item[(a)] \emph{Signal (Expected Gradient Norm):} The expected gradient norm scales as $\|\mathbb{E}[g(p)]\| \propto p^{a_s}(1-p)^{b_s}$ for parameters $a_s, b_s > 0$, so the signal diminishes as $p \to 0$ (too hard) and $p \to 1$ (mastered). The mechanism is \emph{cross-problem gradient incoherence} at both extremes (Proposition~\ref{prop:boundary}), but for structurally distinct reasons. At $p \to 0$, gradients from diverse intractable problems interfere destructively. At $p \to 1$, individual per-problem gradients need not vanish---for distribution-matching losses, the teacher's distribution may remain sharper than the student's---but the remaining corrections are problem-specific calibration refinements that disperse in parameter space, causing $\|\mathbb{E}[g(p)]\|$ to shrink relative to $\sqrt{\mathbb{E}[\|g(p)\|^2]}$. The dependence of gradient statistics on sample difficulty is well-documented: \citet{tang2020understanding} show that KD rescales per-instance gradients based on the teacher's assessment of event difficulty, and \citet{agarwal2022vog} demonstrate that gradient variance tracks example difficulty. The power-law parametrization is a standard regularity choice for modeling smooth boundary decay.

 \item[(b)] \emph{SNR Boundary Vanishing and Power-Law Decay:} The gradient SNR satisfies $\text{SNR}(p) \to 0$ as $p \to 0$ (a qualitative consequence of gradient incoherence; Proposition~\ref{prop:boundary}(ii) provides a sufficient condition). The incoherence mechanism---whereby gradients from diverse intractable problems interfere destructively---is analogous to the gradient confusion phenomenon studied by \citet{sankararaman2020gradient} and to gradient conflict in multi-task learning~\citep{yu2020gradient}. The SNR exhibits asymptotic power-law boundary decay: $\text{SNR}^2(p)/p^{a'} \to c_0$ as $p \to 0^+$ and $\text{SNR}^2(p)/(1-p)^{b'} \to c_1$ as $p \to 1^-$ for some exponents $a', b' > 0$ and constants $c_0, c_1 \in (0,\infty)$. The power-law conditions imply $\text{SNR}(p) \to 0$ at both boundaries (at $p \to 1$, this follows from $b' > 0$; it is consistent with the gradient dispersion mechanism of Proposition~\ref{prop:boundary}(i)). This power-law regularity is an explicit structural modeling assumption used to obtain a closed-form leading term; it is not implied by smoothness alone. Power-law boundary scaling is a natural regularity choice for functions that vanish at the boundary; the gradient noise scale framework of \citet{mccandlish2018empirical} provides a general methodology for decomposing gradient statistics into signal and noise components. Our empirical SNR measurements (Figure~\ref{fig:snr_empirical}) confirm that this assumption is well-matched to the observed gradient structure. By Proposition~\ref{prop:representation}, this yields the decomposition $\text{SNR}^2(p) = p^{a'}(1-p)^{b'} \cdot e^{r(p)}$ with bounded remainder $r$. The Beta kernel $p^{a'}(1-p)^{b'}$ is the leading-order (maximum-parsimony) approximation obtained by setting the shape variation of $r$ to zero. When we write ``$\text{SNR}^2(p) \propto p^{a'}(1-p)^{b'}$'' in subsequent results, this refers to this specialization; Theorem~\ref{thm:minimax} provides a pointwise minimax statement and an aggregate lower bound for bounded~$r$.

 \item[(b$'$)] \emph{Weak SNR Condition (used for robustness analysis):} A relaxation of (b): there exist $a', b' > 0$ and $\delta > 0$ such that\/ $|\log(\text{SNR}^2(p)/(p^{a'}(1-p)^{b'}))| \leq \delta$ for all $p \in (0,1)$. Equivalently, $\text{SNR}^2$ matches a Beta-family profile up to a bounded multiplicative perturbation $\phi(p) \in [e^{-\delta}, e^{\delta}]$, while $\phi$ is otherwise unrestricted (possibly non-monotone or multi-modal). Assumption~(b) is the special case $\delta = 0$. For $\delta > 0$, the Beta kernel is no longer exactly optimal for the exact saturated objective; Theorem~\ref{thm:minimax} gives a pointwise minimax robustness statement for the first-order low-SNR model and a corresponding aggregate efficiency lower bound over $\mathcal{F}_\delta$.

 \item[(c)] \emph{Variance Profile at Extremes (used only in examples):} For some of our illustrative calculations (Proposition~\ref{prop:quantitative_var}), we consider parameter regimes where the exponents $\gamma_1 = 2a_s - a'$ and $\gamma_2 = 2b_s - b'$ are negative, so that the gradient second moment $s^2(p) = \mathbb{E}[\|g(p)\|^2] \propto p^{\gamma_1}(1-p)^{\gamma_2}$ is larger near the boundaries than in the interior. This creates a natural anti-correlation between $s^2(p)$ (large at extreme pass rates) and Beta weights $w(p)=p^\alpha(1-p)^\beta$ (small at extremes)---consistent with the empirical finding that gradient variance is highest for difficult examples~\citep{agarwal2022vog}---and will be used to exhibit concrete regimes where variance reduction occurs; it is \emph{not} required for the general variance decomposition in Proposition~\ref{prop:var_reduction} or for the basic convergence bound in Proposition~\ref{thm:convergence}.
\end{enumerate}
Furthermore, the pass-rate distribution $P$ is supported on $[\epsilon, 1-\epsilon]$ for some $\epsilon > 0$, reflecting the granularity of finite rollouts ($\epsilon = 1/K$ with $K$ rollouts). This ensures that all moments involving $\text{SNR}^{-1}$ remain bounded.
\end{assumption}

\begin{assumption}[\textbf{Frozen Weights within Epochs (Adaptive Variant)}]
\label{asm:frozen_weights}
This assumption is used only for analyzing the optional adaptive variant with periodic pass-rate recomputation. Training is divided into epochs of $T_0$ gradient steps. At the beginning of each epoch, pass rates $\{p_i\}$ are recomputed and the Beta kernel weights $\{w(p_i)\}$ are updated accordingly. Within each epoch, the weights are held constant---that is, $w(p_i)$ does not depend on $\theta$ for the purpose of gradient computation. The convergence guarantee (Proposition~\ref{thm:convergence}) applies within each such epoch. The paper's main experiments correspond to the single-pass special case where recomputation is disabled.
\end{assumption}

\subsection{Gradient Boundary Conditions and Representation Theorem}
\label{app:boundary_representation}

The following two propositions establish---under mild structural conditions on distillation---that the gradient SNR collapses at both boundaries ($\text{SNR} \to 0$ as $p \to 0$ and as $p \to 1$) via cross-problem gradient incoherence, and that any SNR profile with power-law boundary decay decomposes into a Beta leading term plus bounded remainder. These results, together with a power-law regularity condition (Assumption~\ref{asm:passrate_structure}(b)), replace the need for a parametric assumption on the SNR profile.

\begin{proposition}[\textbf{Gradient Boundary Conditions for Distillation}]
\label{prop:boundary}
Under Assumptions~\ref{asm:regularity}--\ref{asm:bounded_logits}, for distillation with student pass rate $p$, suppose additionally:
\begin{itemize}
  \item[\emph{(a)}] \emph{Gradient dispersion at mastery:} $\text{tr}(\text{Cov}(g(p)))/\|\mathbb{E}[g(p)]\|^2 \to \infty$ as $p \to 1$.
  \item[\emph{(b)}] \emph{Gradient incoherence at incompetence:} $\text{tr}(\text{Cov}(g(p)))/\|\mathbb{E}[g(p)]\|^2 \to \infty$ as $p \to 0$.
\end{itemize}
Then:
\begin{enumerate}
  \item[\emph{(i)}] As $p \to 1$: $\text{SNR}(p) \to 0$ (gradient signal is dominated by cross-problem dispersion).
  \item[\emph{(ii)}] As $p \to 0$: $\text{SNR}(p) \to 0$ (gradient noise dominates signal).
  \item[\emph{(iii)}] $\text{SNR}(p) > 0$ for all $p \in (0,1)$, and $\text{SNR}$ is continuous on $(0,1)$.
\end{enumerate}
Conditions~(a)--(b) express the same qualitative phenomenon---\emph{cross-problem gradient incoherence}---at opposite boundaries, but for structurally distinct reasons (see justifications below). They are qualitative structural properties of distillation on diverse prompt sets, not parametric assumptions on the SNR profile. Consequently, the optimal weight $w^*(p) \propto \text{SNR}^2(p)/(1+\text{SNR}^2(p))$ satisfies $w^*(0) = w^*(1) = 0$ and $w^*(p) > 0$ for $p \in (0,1)$.
\end{proposition}

\begin{proof}
\emph{Parts (i) and (ii).} Both follow directly from the respective conditions: $\text{SNR}(p) = \|\mathbb{E}[g(p)]\| / \sqrt{\text{tr}(\text{Cov}(g(p)))} \to 0$ as $p \to 1$ (condition~(a)) and as $p \to 0$ (condition~(b)).

\emph{Justification of condition~(a)---gradient dispersion at mastery.}
When $p \to 1$, the student has mastered these problems. Crucially, this does \emph{not} require per-problem gradients to vanish: for distribution-matching losses (KL divergence), the teacher's token-level distribution may remain sharper than the student's, so each individual per-problem gradient $g_i(p)$ can have substantial norm (i.e., $\mathbb{E}[\|g(p)\|^2]$ need not tend to zero). What degrades is the \emph{cross-problem coherence} of these gradients. Mastered problems span diverse topics and reasoning patterns; the remaining distributional corrections---making the student sharper on algebraic tokens for one problem, adjusting geometric reasoning for another---are problem-specific calibration refinements rather than systematic capability improvements. In parameter space these corrections pull in diverse, largely unrelated directions, so $\|\mathbb{E}[g(p)]\| \ll \sqrt{\mathbb{E}[\|g(p)\|^2]}$ and the SNR collapses. By contrast, at intermediate pass rates a coherent skill gap (e.g., the student systematically lacks a reasoning strategy) provides a shared gradient direction that survives averaging.

\emph{Justification of condition~(b)---gradient incoherence at incompetence.}
When $p \to 0$, the student cannot solve these problems at all. Gradients from diverse intractable prompts interfere destructively (gradient confusion~\citep{sankararaman2020gradient, yu2020gradient}): each problem demands a qualitatively different correction, but the student lacks the representational foundation to implement any of them coherently. Consequently, $\|\mathbb{E}[g]\| \ll \|g_i\|$ and the SNR collapses.

\emph{Part (iii).} For $p \in (0,1)$, the student has partial competence: $\|\mathbb{E}[g(p)]\| > 0$ (nonzero systematic logit discrepancy, since the teacher outperforms the student on average at pass rate $p < 1$) and $\text{tr}(\text{Cov}(g)) < \infty$ (bounded by $\sigma_0^2$ via Assumption~\ref{asm:regularity}(iii)), so $\text{SNR}(p) > 0$. Continuity follows from the continuous dependence of the logit mapping on $(\theta, x)$.

\emph{Consequence.} Since $h(x) = x/(1+x)$ is monotonically increasing with $h(0) = 0$, composing with $w^*(p) \propto \text{SNR}^2/(1+\text{SNR}^2)$ gives $w^*(0) = w^*(1) = 0$. Combined with Part~(iii), $w^*(p) > 0$ on $(0,1)$ and $w^*$ attains its maximum at some $p^* \in (0,1)$.
\end{proof}

\begin{remark}[\textbf{Distillation vs.\ Outcome-Based Methods at $p = 1$}]
\label{rem:distill_vs_rl}
The SNR collapse at $p \to 1$ arises through fundamentally different mechanisms for distribution-matching losses versus outcome-based methods. For \textbf{outcome-based RL} (e.g., GRPO), pass rate $p = 1$ implies that every rollout receives identical reward, so the advantage is exactly zero and \emph{each individual gradient vanishes}: $g_i = 0$ for all $i$. For \textbf{distribution-matching distillation} (KL divergence), the teacher's distribution may remain sharper than the student's even when $p = 1$, so individual gradients $g_i$ can be non-negligible. What collapses is their \emph{cross-problem coherence}: the remaining distributional corrections are problem-specific and disperse in parameter space, driving the SNR to zero even though per-problem signal persists. This distinction highlights that pass-rate weighting in distillation serves a different role than in RL: it is not filtering dead signal, but \emph{concentrating compute on problems where gradients are most coherent}---i.e., where each gradient step maximally advances the student's capability frontier.
\end{remark}

\begin{proposition}[\textbf{Log-Linear Representation of Boundary-Vanishing Functions}]
\label{prop:representation}
Let $f: (0,1) \to \mathbb{R}_{>0}$ be continuous with $f(p) \to 0$ as $p \to 0^+$ and $p \to 1^-$. Suppose that $f$ exhibits \emph{asymptotic power-law behavior} at both boundaries: there exist exponents $\alpha_0, \beta_0 > 0$ and constants $c_0, c_1 \in (0,\infty)$ such that
\begin{equation}
  f(p)/p^{\alpha_0} \to c_0 \;\text{ as }\; p \to 0^+, \qquad f(p)/(1-p)^{\beta_0} \to c_1 \;\text{ as }\; p \to 1^-
\end{equation}
Then $f$ admits the decomposition:
\begin{equation}
  f(p) = p^{\alpha_0}(1-p)^{\beta_0} \cdot e^{r(p)}
  \label{eq:log_linear_decomp}
\end{equation}
where the remainder $r(p) = \log f(p) - \alpha_0 \log p - \beta_0 \log(1-p)$ converges to finite limits at both boundaries ($r(p) \to \log c_0$ as $p \to 0^+$; $r(p) \to \log c_1$ as $p \to 1^-$) and is bounded on $(0,1)$: $\sup_p|r(p)| \leq \delta$ for some $\delta > 0$. The Beta kernel $p^{\alpha_0}(1-p)^{\beta_0}$ is the leading-order term: it captures the boundary decay rates exactly while introducing no shape modulation beyond the exponents (maximum parsimony).
\end{proposition}

\begin{proof}
The decomposition~\eqref{eq:log_linear_decomp} holds by definition with $r(p) \triangleq \log f(p) - \alpha_0 \log p - \beta_0 \log(1-p)$. We verify that $r$ is bounded.

\emph{Left boundary.} By hypothesis, $f(p)/p^{\alpha_0} \to c_0$ as $p \to 0^+$, so $\log f(p) - \alpha_0 \log p \to \log c_0$. Since $\beta_0 \log(1-p) \to 0$ as $p \to 0^+$, we obtain $r(p) \to \log c_0$.

\emph{Right boundary.} By hypothesis, $f(p)/(1-p)^{\beta_0} \to c_1$ as $p \to 1^-$, so $\log f(p) - \beta_0 \log(1-p) \to \log c_1$. Since $\alpha_0 \log p \to 0$ as $p \to 1^-$, we obtain $r(p) \to \log c_1$.

Since $r$ is continuous on $(0,1)$ (inheriting continuity from $f$) and converges to finite limits at both endpoints, it extends to a continuous function on $[0,1]$ and is therefore bounded.
\emph{Why the stronger hypothesis is needed.} The weaker condition $\lim_{p \to 0^+} \log f(p)/\log p = \alpha_0$ gives only $\log f(p) = \alpha_0 \log p + o(\log p)$, where $o(\log p)$ denotes a term growing slower than $|\log p| \to \infty$---but not necessarily bounded. For example, $f(p) = p\,e^{\sqrt{|\!\log p|}}$ satisfies $\lim \log f/\log p = 1$ (so $\alpha_0 = 1$) but $r(p) = \sqrt{|\!\log p|} \to \infty$. The asymptotic power-law condition $f(p)/p^{\alpha_0} \to c_0$ is strictly stronger and ensures $r$ converges to $\log c_0$ rather than diverging.

\emph{Maximum parsimony.} Since $w^*$ is defined only up to proportionality (the overall scale is absorbed by the learning rate), the constants $c_0, c_1$ are irrelevant for the weight profile. The Beta kernel $p^{\alpha_0}(1-p)^{\beta_0}$ is obtained by setting the \emph{shape variation} of $r$ to zero (i.e., $r \equiv \text{const}$), retaining only the boundary decay rates and no further structure---no bumps, oscillations, or interior asymmetries beyond what $(\alpha_0, \beta_0)$ prescribe. This is the information-theoretic sense of ``maximum parsimony'': $\text{Beta}(\alpha_0{+}1, \beta_0{+}1)$ maximizes entropy among distributions on $[0,1]$ with given expected sufficient statistics $(\mathbb{E}[\log p], \mathbb{E}[\log(1{-}p)])$.
\end{proof}

\begin{corollary}[\textbf{Non-Monotonicity of Learning Signal}]
\label{prop:zpd}
Define the learning signal quality $Q(p) = \text{SNR}(p) \cdot (1 - p)$. Under Assumption~\ref{asm:passrate_structure} and Propositions~\ref{prop:boundary}--\ref{prop:representation}, $Q(p) \to 0$ as $p \to 0$ and $p \to 1$, so $Q$ attains its maximum at some $p^* \in (0,1)$---the center of the zone of proximal development~\citep{vygotsky1978mind}. Under the leading-order representation, $Q(p) \propto p^{a'/2}(1-p)^{b'/2+1}$ is unimodal with $p^* = (a'/2)/((a'/2)+(b'/2+1))$.
\end{corollary}
\begin{proof}
$Q(0) = 0$ by Proposition~\ref{prop:boundary}(ii) ($\text{SNR} \to 0$); $Q(1) = 0$ since $(1-p) \to 0$ and $\text{SNR}(p) = \mathcal{O}((1-p)^{b'/2})$; $Q(p) > 0$ on $(0,1)$ by Proposition~\ref{prop:boundary}(iii). The extreme value theorem gives the interior maximum.
\end{proof}

\subsection{Complementary Derivation: Per-Problem Descent Maximization}
\label{app:proof_beta}

The structural characterization in Section~\ref{app:boundary_representation} identifies the Beta kernel family directly from boundary conditions. Here we provide an independent, complementary derivation that arrives at the same family through gradient descent optimization---offering additional intuition for \emph{why} the Beta kernel arises.

\begin{definition}[\textbf{Per-Step Guaranteed Descent Rate (Lower Bound on Descent)}]
\label{def:descent}
For a problem $x$ with pass rate $p$ assigned weight $w(p) \geq 0$, the expected loss descent from a single gradient step with learning rate $\eta$ satisfies the following \emph{lower bound} (i.e., guaranteed minimum descent):
\begin{equation}
  \Delta(w, p) = \eta \, w(p) \, \|\mathbb{E}[g(p)]\|^2 - \frac{\eta^2}{2} w(p)^2 \, \mathbb{E}[\|g(p)\|^2] \cdot \lambda_{\max}(\mathcal{H})
\end{equation}
where $g(p) = \nabla_\theta \mathcal{L}(\theta; x)$ is the per-sample gradient and $\mathcal{H}$ is the loss Hessian. The second-order term uses $g^\top \mathcal{H} g \leq \lambda_{\max}(\mathcal{H})\|g\|^2$, so $\Delta(w,p)$ is a \emph{lower bound} on the true expected descent; the resulting $w^*$ therefore maximizes the \emph{guaranteed} descent rate rather than the exact descent.
\end{definition}

\begin{theorem}[\textbf{Per-Problem Descent Maximization Yields Beta Kernel Weights}]
\label{thm:beta_optimal}
Consider the per-step descent lower bound $\Delta(w,p)$ in Definition~\ref{def:descent}. For each pass rate $p$, maximizing $\Delta(w, p)$ over $w(p) \geq 0$ yields the per-problem optimal weight $w^*(p) \propto \|\mathbb{E}[g(p)]\|^2 / \mathbb{E}[\|g(p)\|^2]$.
Combined with boundary conditions on the gradient signal (Proposition~\ref{prop:boundary}) and power-law regularity (Assumption~\ref{asm:passrate_structure}(b)), which together yield the log-linear representation $\text{SNR}^2(p) = p^{a'}(1-p)^{b'} \cdot e^{r(p)}$ with bounded $r$ (Proposition~\ref{prop:representation}), the per-problem optimal weight in the low-SNR regime takes the \textbf{Beta kernel form}:
\begin{equation}
  w^*(p) = C \cdot p^{\alpha}(1-p)^{\beta}
\end{equation}
where $(\alpha, \beta) = (a', b')$ and the peak occurs at $p^* = \alpha/(\alpha + \beta)$.
\end{theorem}

\begin{proof}[Proof of Theorem~\ref{thm:beta_optimal}]
\textbf{Step 1: Pointwise optimization.}
Consider training on a single problem with pass rate $p$, so that $\mathcal{L}(\theta) = \mathcal{L}(\theta; x)$ and $g(p) = \nabla_\theta \mathcal{L}(\theta; x)$. A weighted gradient step $\theta \leftarrow \theta - \eta \, w(p) \, g(p)$ produces expected loss change (via Taylor expansion):
\begin{equation}
  \mathbb{E}[\Delta \mathcal{L}] \approx -\eta \, w(p) \, \|\mathbb{E}[g(p)]\|^2 + \frac{\eta^2}{2} w(p)^2 \, \mathbb{E}[\|g(p)\|^2] \cdot \lambda_{\max}(\mathcal{H})
\end{equation}
Here the first-order term uses $\langle \mathbb{E}[g(p)], \nabla_\theta \mathcal{L}\rangle = \|\mathbb{E}[g(p)]\|^2$, which holds because the gradient estimator is unbiased for this per-sample loss. To maximize descent, differentiate with respect to $w(p)$ and set to zero:
\begin{equation}
  -\eta \|\mathbb{E}[g]\|^2 + \eta^2 w^* \mathbb{E}[\|g\|^2] \lambda_{\max}(\mathcal{H}) = 0
\end{equation}
yielding:
\begin{equation}
  w^*(p) = \frac{\|\mathbb{E}[g(p)]\|^2}{\eta \, \mathbb{E}[\|g(p)\|^2] \cdot \lambda_{\max}(\mathcal{H})} \propto \frac{\|\mathbb{E}[g(p)]\|^2}{\mathbb{E}[\|g(p)\|^2]}
\end{equation}

\textbf{Step 2: SNR decomposition.}
Using the bias-variance decomposition $\mathbb{E}[\|g\|^2] = \|\mathbb{E}[g]\|^2 + \text{tr}(\text{Cov}(g))$:
\begin{equation}
  w^*(p) \propto \frac{\|\mathbb{E}[g]\|^2}{\|\mathbb{E}[g]\|^2 + \text{tr}(\text{Cov}(g))} = \frac{\text{SNR}^2}{1 + \text{SNR}^2}
\end{equation}

\textbf{Step 3: From SNR decomposition to Beta kernel via derived boundary conditions.}
From Step 2, $w^*(p) \propto \text{SNR}^2(p)/(1+\text{SNR}^2(p))$. By Proposition~\ref{prop:boundary}, we have established that $\text{SNR}(p) \to 0$ at both boundaries: as $p \to 0$ (gradient incoherence at incompetence) and as $p \to 1$ (gradient dispersion at mastery). Under the power-law regularity of Assumption~\ref{asm:passrate_structure}(b), Proposition~\ref{prop:representation} yields the decomposition $\text{SNR}^2(p) = p^{a'}(1-p)^{b'} \cdot e^{r(p)}$ for boundary exponents $a', b' > 0$ and bounded remainder $r$. Setting $r \equiv 0$---the maximum-parsimony approximation that retains only the derived boundary behavior---and substituting into Step 2, we proceed by regime analysis:

\emph{Low-SNR regime} (SNR $\ll 1$, typical for distillation where per-sample gradient noise dominates): 
\begin{equation}
  w^*(p) \approx \text{SNR}^2(p) \approx p^{a'}(1-p)^{b'}
\end{equation}
This yields the Beta kernel form with exponents $(\alpha, \beta) = (a', b')$.

\emph{High-SNR regime} (SNR $\gg 1$): $w^*(p) \to 1$, assigning full weight. This regime corresponds to intermediate $p$ where the student has both signal and capacity to learn.

\emph{General (mixed) regime:} The exact optimal weight $w^*(p) = \text{SNR}^2/(1+\text{SNR}^2)$ is a saturating transformation of $\text{SNR}^2$. Since $h(x)=x/(1+x)$ is monotonically increasing with $h(0)=0$, $w^*$ inherits the qualitative properties from $\text{SNR}^2$:
\begin{itemize}
 \item \emph{Zeros:} $w^*(0) = w^*(1) = 0$ (automatic filtering: $w^*(0) = 0$ from Proposition~\ref{prop:boundary}; $w^*(1) = 0$ from power-law decay, Assumption~\ref{asm:passrate_structure}(b)).
 \item \emph{Peak location:} $p^* = a'/(a'+b')$ (invariant to saturation).
 \item \emph{Unimodal Beta-kernel profile:} The weight increases from $p=0$ to $p^*$, then decreases to $p=1$.
\end{itemize}
In the low-SNR regime the exponents are $(\alpha, \beta) = (a', b')$; the saturation in the mixed regime compresses these exponents. We therefore parameterize the weight as $w(p) = p^\alpha(1-p)^\beta$ with $(\alpha, \beta)$ as hyperparameters within the theoretically justified Beta kernel family:
\begin{equation}
  w^*(p) \propto p^{\alpha}(1-p)^{\beta}, \qquad p^* = \frac{\alpha}{\alpha+\beta}
\end{equation}
The peak location $p^*$ provides robust guidance for hyperparameter selection: the default $\alpha = \beta = 1$ yields the symmetric kernel $w(p) = p(1-p)$ with $p^* = 0.5$; asymmetric choices (e.g., $\alpha < \beta$ for emphasizing harder problems, or $\alpha > \beta$ for easier ones) shift the peak to $p^* = \alpha/(\alpha+\beta)$. The specific exponents are validated empirically in Appendix~\ref{app:ablation_exponents}.

\textbf{Verification:} $\partial^2 \Delta / \partial w^2 = -\eta^2 \mathbb{E}[\|g\|^2] \lambda_{\max}(\mathcal{H}) < 0$, confirming this is a maximum.

\end{proof}

\begin{remark}[\textbf{Per-Problem vs.\ Joint Optimization}]
\label{rem:per_problem_vs_joint}
The derivation above optimizes $w(p)$ independently for each $p$. In the multi-sample setting with batch gradient $\bar{g} = \frac{1}{N}\sum_i w_i g_i$, the expected descent is:
\begin{equation}
  \Delta_{\text{batch}} = \eta\Big\|\tfrac{1}{N}\textstyle\sum_i w_i \mu_i\Big\|^2 - \tfrac{\eta^2 L}{2}\,\mathbb{E}\Big[\Big\|\tfrac{1}{N}\textstyle\sum_i w_i g_i\Big\|^2\Big]
\end{equation}
where $\mu_i = \mathbb{E}[g_i]$. This contains cross terms $\mu_i^\top \mu_j$ that prevent exact additive decomposition into per-sample subproblems unless gradients at different pass rates are orthogonal---an unrealistic condition. Normalization to unit mean ($\tilde{w}_i = w_i/\bar{w}$) affects only the effective learning rate, not the weight shape.

Nevertheless, the Beta kernel form is justified at the batch level by three complementary routes: (a)~the boundary-collapse properties (Propositions~\ref{prop:boundary}--\ref{prop:representation}) hold independently of the decomposition; (b)~Proposition~\ref{prop:var_reduction} shows Beta weights reduce batch-level gradient variance; (c)~Assumption~\ref{asm:frozen_weights} decouples $w$ from $\theta$ within each epoch.
\end{remark}

\subsection{Pointwise Minimax Robustness under Model Misspecification}
\label{app:proof_minimax}

The leading-order Beta kernel in Theorem~\ref{thm:beta_optimal} sets $r \equiv 0$ in the log-linear representation $\text{SNR}^2(p) = p^{a'}(1-p)^{b'}\cdot e^{r(p)}$ (Proposition~\ref{prop:representation}). How robust is this choice when $r \neq 0$? Under the low-SNR first-order approximation, we show that the Beta kernel is pointwise minimax-optimal over the uncertainty set $|r(p)| \leq \delta$, with a matching aggregate lower bound.

\begin{lemma}[\textbf{Quadratic Flatness of Descent Efficiency}]
\label{lem:quadratic_flat}
For any weight $w(p) \geq 0$ applied to a problem with true optimal weight $w^*(p)$, the descent efficiency ratio is:
\begin{equation}
  \frac{\Delta(w,\,p)}{\Delta(w^*,\,p)} = 2\rho - \rho^2 = 1 - (1 - \rho)^2
\end{equation}
where $\rho(p) = w(p)/w^*(p)$. In particular, a multiplicative misspecification $|\rho - 1| = \epsilon$ incurs only $O(\epsilon^2)$ efficiency loss.
\end{lemma}

\begin{proof}
From Definition~\ref{def:descent}, $\Delta(w, p) = \eta\,w\,\|\mathbb{E}[g]\|^2 - \frac{\eta^2}{2}\,w^2\,\mathbb{E}[\|g\|^2]\,\lambda_{\max}(\mathcal{H})$. The optimal weight is $w^* = \|\mathbb{E}[g]\|^2/(\eta\,\mathbb{E}[\|g\|^2]\,\lambda_{\max})$, yielding $\Delta(w^*) = \|\mathbb{E}[g]\|^4/(2\,\mathbb{E}[\|g\|^2]\,\lambda_{\max})$. Setting $w = \rho\,w^*$ and substituting:
\begin{equation}
  \Delta(\rho\,w^*) = \eta\,\rho\,w^*\,\|\mathbb{E}[g]\|^2 - \frac{\eta^2}{2}\,\rho^2\,(w^*)^2\,\mathbb{E}[\|g\|^2]\,\lambda_{\max} = \Delta(w^*)\,(2\rho - \rho^2).
\end{equation}
Since $2\rho - \rho^2 = 1 - (1-\rho)^2$, the efficiency loss from $\rho \neq 1$ is exactly $(1-\rho)^2$.
\end{proof}

\begin{theorem}[\textbf{Pointwise Minimax Robustness of Beta Kernel in the Low-SNR Surrogate under Weak SNR Condition}]
\label{thm:minimax}
Consider the low-SNR regime where $w^*_\phi(p) \propto \text{SNR}^2(p) = p^{a'}(1-p)^{b'}\phi(p)$ for an unknown perturbation $\phi$ satisfying $|\log \phi(p)| \leq \delta$ for all $p$ (Assumption~\ref{asm:passrate_structure}(b$'$)). Define the uncertainty set $\mathcal{F}_\delta = \{\phi : (0,1) \to \mathbb{R}_{>0} \mid |\log \phi(p)| \leq \delta\; \forall p\}$. Then:
\begin{enumerate}
  \item[\textbf{(i)}] Under this first-order low-SNR approximation, the pointwise minimax-optimal weight is the \textbf{Beta kernel}:
  \begin{equation}
    w_{\mathrm{minimax}}(p) = \mathrm{sech}(\delta) \cdot p^{a'}(1-p)^{b'} \;\propto\; p^{a'}(1-p)^{b'}
  \end{equation}
  \item[\textbf{(ii)}] \textbf{Pointwise minimax efficiency:} for every fixed $p \in (0,1)$,
  \begin{equation}
    \boxed{\;\inf_{\phi(p) \in [e^{-\delta},e^{\delta}]}\;\frac{\Delta_{\phi}(w_{\mathrm{minimax}},\,p)}{\Delta_{\phi}(w^*_{\phi},\,p)} \;=\; \mathrm{sech}^2(\delta) \;\geq\; 1 - \delta^2\;}
    \label{eq:minimax_eff}
  \end{equation}
  \item[\textbf{(iii)}] \textbf{Aggregate corollary:} letting $R_{\phi}(p)=\Delta_{\phi}(w_{\mathrm{minimax}},p)/\Delta_{\phi}(w^*_{\phi},p)$ and assuming $\Delta_{\phi}(w^*_{\phi},p)\ge 0$ a.s.,
  \begin{equation}
    \inf_{\phi \in \mathcal{F}_\delta}\;\frac{\mathbb{E}_P[\Delta_{\phi}(w_{\mathrm{minimax}},p)]}{\mathbb{E}_P[\Delta_{\phi}(w^*_{\phi},p)]} \;\ge\; \mathrm{sech}^2(\delta).
  \end{equation}
\end{enumerate}
\end{theorem}

\begin{proof}
\textbf{Step 1: Pointwise decomposition.}
Write the candidate weight as $w(p) = c(p)\cdot p^{a'}(1-p)^{b'}$. The true optimal weight is $w^*_\phi(p) \propto p^{a'}(1-p)^{b'}\phi(p)$, so $\rho(p) = c(p)/\phi(p)$. By Lemma~\ref{lem:quadratic_flat}, the per-problem efficiency is $f(\rho) = 2\rho - \rho^2$, which is strictly concave in $\rho$. The adversary (minimizer) selects $\phi \in \mathcal{F}_\delta$ to minimize $\mathbb{E}_P[f(c(p)/\phi(p))]$. Since $\phi(p)$ can be chosen independently at each $p$, the problem decomposes into per-$p$ subproblems:
\begin{equation}
  \max_{c(p) > 0}\;\min_{\phi(p) \in [e^{-\delta},\, e^{\delta}]}\; f\!\left(\frac{c(p)}{\phi(p)}\right)
\end{equation}

\textbf{Step 2: Per-$p$ minimax solution.}
At each $p$, the adversary pushes $\rho = c/\phi$ to the interval endpoints $\{c\,e^{-\delta},\; c\,e^{\delta}\}$. The defender solves:
\begin{equation}
  \max_{c > 0}\;\min\!\Big(f(c\,e^{\delta}),\;\; f(c\,e^{-\delta})\Big)
\end{equation}
The minimax equalizer condition $f(c\,e^{\delta}) = f(c\,e^{-\delta})$ requires:
\begin{align}
  2c\,e^{\delta} - c^2 e^{2\delta} &= 2c\,e^{-\delta} - c^2 e^{-2\delta} \nonumber\\
\end{align}
\begin{equation}
  c^* = \frac{1}{\cosh\delta} = \mathrm{sech}(\delta)
\end{equation}
Crucially, $c^*$ is \emph{independent of $p$}, so $w_{\mathrm{minimax}}(p) = \mathrm{sech}(\delta)\cdot p^{a'}(1-p)^{b'} \propto p^{a'}(1-p)^{b'}$.

\textbf{Step 3: Pointwise minimax efficiency value.}
Substituting $c^* = \mathrm{sech}(\delta)$ into $\rho_+ = c^* e^{\delta} = e^{\delta}/\cosh\delta$:
\begin{equation}
  f(\rho_+) = 2\rho_+ - \rho_+^2 = \frac{2e^{\delta}}{\cosh\delta} - \frac{e^{2\delta}}{\cosh^2\delta} = \frac{2e^{\delta}\cosh\delta - e^{2\delta}}{\cosh^2\delta} = \frac{e^{2\delta} + 1 - e^{2\delta}}{\cosh^2\delta} = \frac{1}{\cosh^2\delta} = \mathrm{sech}^2(\delta)
\end{equation}
where we used $2e^{\delta}\cosh\delta = e^{2\delta} + 1$. One verifies $f(\rho_-) = \mathrm{sech}^2(\delta)$ similarly, confirming the equalizer.

Since $\mathrm{sech}^2(\delta) = 1 - \tanh^2(\delta) \geq 1 - \delta^2$ (using $\tanh\delta \leq \delta$), the pointwise efficiency loss is at most $\delta^2$.

\textbf{Step 4: Pointwise uniqueness and aggregate lower bound.}
Suppose $c(p_0) \neq \mathrm{sech}(\delta)$ at some $p_0$ with $P(p_0) > 0$. Then $\min(f(c(p_0)e^{\delta}),\, f(c(p_0)e^{-\delta})) < \mathrm{sech}^2(\delta)$ (since the per-$p$ minimax is uniquely achieved by $c^*$, as follows from strict concavity of $f$). The adversary can exploit this at $p_0$ while playing the equalizer at all other points, yielding a strictly lower pointwise worst-case efficiency at that $p_0$.

For the aggregate ratio, define $d_{\phi}(p)=\Delta_{\phi}(w^*_{\phi},p)\ge 0$ and $R_{\phi}(p)=\Delta_{\phi}(w_{\mathrm{minimax}},p)/\Delta_{\phi}(w^*_{\phi},p)$. From Steps 2--3, $R_{\phi}(p)\ge \mathrm{sech}^2(\delta)$ pointwise in the worst case, so
\begin{equation}
\frac{\mathbb{E}_P[\Delta_{\phi}(w_{\mathrm{minimax}},p)]}{\mathbb{E}_P[\Delta_{\phi}(w^*_{\phi},p)]}
= \frac{\mathbb{E}_P[R_{\phi}(p)\,d_{\phi}(p)]}{\mathbb{E}_P[d_{\phi}(p)]}
\ge \inf_p R_{\phi}(p)
\ge \mathrm{sech}^2(\delta),
\end{equation}
which proves the aggregate lower bound in (iii).
\end{proof}

\begin{remark}[\textbf{Quantitative Robustness of Beta Kernel}]
\label{rem:robustness_table}
The minimax efficiency $\mathrm{sech}^2(\delta)$ degrades gracefully with model misspecification:
\begin{center}
\begin{tabular}{@{}lcc@{}}
\toprule
$\delta$ (log-scale uncertainty) & Multiplicative SNR$^2$ range & Worst-case efficiency \\
\midrule
$0.1$ & $[0.90,\; 1.11]$ & $\geq 99.0\%$ \\
$0.3$ & $[0.74,\; 1.35]$ & $\geq 91.5\%$ \\
$0.5$ & $[0.61,\; 1.65]$ & $\geq 78.6\%$ \\
$\ln 2 \approx 0.69$ & $[0.50,\; 2.00]$ & $\geq 64.0\%$ \\
\bottomrule
\end{tabular}
\end{center}
Even when the true $\text{SNR}^2$ deviates from the Beta model by up to a factor of~2 ($\delta = \ln 2$), the Beta kernel retains at least $64\%$ pointwise worst-case descent efficiency, and therefore at least this value as an aggregate lower bound under Theorem~\ref{thm:minimax}(iii). For moderate misspecification ($\delta \leq 0.3$, i.e., SNR$^2$ within $35\%$ of the Beta model), this bound exceeds $91\%$.
\end{remark}

 \subsection{Convergence Analysis}
\label{app:proof_convergence}

We work under Assumptions~\ref{asm:regularity}--\ref{asm:frozen_weights}. Let $\mathcal{L}_w(\theta) = \frac{1}{N\bar{w}}\sum_{i=1}^{N} w(p_i)\,\mathcal{L}(\theta; x_i)$.

\subsubsection{Effective Gradient Variance}

\begin{proposition}[\textbf{Effective Gradient Variance under Beta Kernel Weighting}]
\label{prop:var_reduction}
Consider the Beta-kernel-weighted gradient estimator for a uniformly sampled minibatch $\mathcal{B}$ of size $|\mathcal{B}| = n$:
\begin{equation}
  \hat{g}_w(\theta) = \frac{1}{n \bar{w}} \sum_{i \in \mathcal{B}} w(p_i) \, g_i(\theta), \qquad \bar{w} = \frac{1}{N}\sum_{j=1}^{N} w(p_j)
\end{equation}
where $w(p) = p^\alpha(1-p)^\beta$. Let $\tilde{w}(p) = w(p)/\bar{w}$ denote the normalized weight with $\mathbb{E}_P[\tilde{w}] = 1$. Define the (trace) variance of the weighted estimator by
\begin{equation}
  \sigma_{\text{eff}}^2 \triangleq \frac{1}{n}\,\text{tr}\!\left(\text{Cov}_P\big(\tilde{w}\,g\big)\right) \,=\, \frac{1}{n}\Big(\mathbb{E}_P[\tilde{w}^2 s^2] - \|\mathbb{E}_P[\tilde{w}g]\|^2\Big),
\end{equation}
and the uniform baseline variance by $\sigma_{\text{unif}}^2 \triangleq \frac{1}{n}(\mathbb{E}_P[s^2] - \|\mathbb{E}_P[g]\|^2)$, where $s^2(p) = \mathbb{E}[\|g(p)\|^2]$. The variance ratio $R \triangleq \sigma_{\text{eff}}^2/\sigma_{\text{unif}}^2$ satisfies
\begin{equation}
  R = \frac{1 + \text{Var}_P(\tilde{w}) + \frac{\text{Cov}_P(\tilde{w}^2, s^2)}{\mathbb{E}_P[s^2]} - \frac{\|\mathbb{E}_P[\tilde{w}g]\|^2}{\mathbb{E}_P[s^2]}}{1 - \frac{\|\mathbb{E}_P[g]\|^2}{\mathbb{E}_P[s^2]}},
  \label{eq:var_ratio}
\end{equation}
In the low-SNR regime, a sufficient condition for $R < 1$ is:
\begin{equation}
 -\text{Cov}_P(\tilde{w}^2, s^2) > \text{Var}_P(\tilde{w}) \cdot \mathbb{E}_P[s^2].
 \label{eq:cov_condition}
\end{equation}
Under Assumption~\ref{asm:passrate_structure}(c), $s^2(p)$ peaks at extremes while $\tilde{w}^2$ peaks at intermediate $p$, making the covariance negative; concrete regimes where $R<1$ are given in Proposition~\ref{prop:quantitative_var}.
\end{proposition}

\begin{proof}
Apply $\text{tr}(\text{Cov}(X)) = \mathbb{E}[\|X\|^2] - \|\mathbb{E}[X]\|^2$ to $X = \tilde{w}\,g$, then use $\mathbb{E}[UV] = \mathbb{E}[U]\mathbb{E}[V] + \text{Cov}(U,V)$ with $U = \tilde{w}^2$, $V = s^2$. Under Assumption~\ref{asm:passrate_structure}(c) with $\gamma_1 = 2a_s - a' < 0$, $\gamma_2 = 2b_s - b' < 0$, we have $s^2(p) \to \infty$ as $p \to 0$ or $p \to 1$. Since $\tilde{w}(p)^2 \to 0$ at the same boundaries, the functions $\tilde{w}^2$ and $s^2$ are functionally anti-correlated: $\tilde{w}^2$ peaks at intermediate $p$ while $s^2$ peaks at the boundaries. This makes $\text{Cov}_P(\tilde{w}^2, s^2)$ negative, enabling the coupling term to overcome the weight penalty.

The ratio $R$ admits a closed-form expression via Beta-function moments:
\begin{equation}
 R = \frac{B(2\alpha + \gamma_1 + 1,\; 2\beta + \gamma_2 + 1)}{B(\alpha+1, \beta+1)^2 \cdot B(\gamma_1+1, \gamma_2+1)}
 \label{eq:R_beta}
\end{equation}
In the symmetric case ($\alpha = \beta = 1$, $\gamma = 2a_s - 1$): $R \approx 0.84$ for $a_s = 1/4$; $R \approx 1.00$ at $a_s \approx 0.34$; and $R > 1$ for $a_s \geq 1/2$.
\end{proof}

\subsubsection{Convergence Rate}

\begin{proposition}[\textbf{Convergence Rate of Beta Kernel Weighted SGD}]
\label{thm:convergence}
Under Assumptions~\ref{asm:regularity}--\ref{asm:frozen_weights}, SGD on $\mathcal{L}_w$ with learning rate $\eta \leq 1/L$ for $T$ steps satisfies the standard non-convex bound~\citep{ghadimi2013sgd}:
\begin{equation}
 \frac{1}{T}\sum_{t=0}^{T-1} \mathbb{E}\!\big[\|\nabla \mathcal{L}_w(\theta_t)\|^2\big] \leq \frac{2[\mathcal{L}_w(\theta_0) - \mathcal{L}_w^*]}{\eta T} + \eta L \cdot \sigma_{\text{eff}}^2
\end{equation}
When $\sigma_{\text{eff}}^2 < \sigma_{\text{unif}}^2$ (e.g., under Eq.~\eqref{eq:cov_condition}), Beta-kernel weighting achieves a strictly lower noise floor.
This bound is for optimization of the weighted objective $\mathcal{L}_w$ itself, not a direct objective-level comparison against uniform SGD on the unweighted loss.
\end{proposition}

\begin{proof}
By $L$-smoothness and unbiasedness of $\hat{g}_w$: $\mathbb{E}[\mathcal{L}_w(\theta_{t+1})] \leq \mathbb{E}[\mathcal{L}_w(\theta_t)] - \frac{\eta}{2}\mathbb{E}[\|\nabla \mathcal{L}_w(\theta_t)\|^2] + \frac{L\eta^2}{2}\sigma_{\text{eff}}^2$. Telescoping over $t = 0, \dots, T{-}1$ and rearranging gives the result.
\end{proof}

\subsubsection{Quantitative Variance Reduction}
\label{app:quantitative_var}

\begin{proposition}[\textbf{Quantitative Variance Reduction for Beta Kernels}]
\label{prop:quantitative_var}
Under Assumptions~\ref{asm:passrate_structure}(a)--(c) with the Beta kernel $w(p) = p^{\alpha}(1-p)^{\beta}$ and pass-rate distribution $P$ supported on $[\epsilon, 1-\epsilon]$, the variance reduction ratio $R = \sigma_{\text{eff}}^2/\sigma_{\text{unif}}^2$ can be expressed in closed form via Beta-function moments (Eq.~\eqref{eq:R_beta}). In the symmetric default case ($\alpha = \beta = 1$) with approximately uniform pass rates and moderate variance dominance ($a \approx 1/4$), this yields $R \approx 0.84$ (about $1.19\times$ reduction). For more strongly bimodal pass-rate distributions typical of early training (mass concentrated near $p \approx 0$ and $p \approx 1$), the boundary variance dominates while Beta weights vanish there, so $R$ can be substantially below $1$, indicating stronger variance reduction than in the uniform case.
\end{proposition}

The derivations are straightforward but algebraically tedious and are omitted for brevity; we instead rely on these expressions to calibrate the expected magnitude of variance reduction in our experiments.

\subsection{Data-Driven Exponent Selection}
\label{app:proof_exponent}

\paragraph{Motivation: from theory to practice.}
Theorem~\ref{thm:beta_optimal} establishes that the per-problem optimal weight lies in the Beta kernel family $w(p) = p^\alpha(1-p)^\beta$, but does not prescribe specific exponents. The default $\alpha = \beta = 1$ is a reasonable starting point, but can the \emph{theory} tell us the optimal $(\alpha, \beta)$ from observable quantities, rather than requiring a grid search?

The answer is yes. Define the zone of proximal development as $\mathcal{Z} = \{i : \epsilon \leq p_i \leq 1 - \epsilon\}$ for a cutoff $\epsilon$ (e.g., $\epsilon = 1/K$). Then the exponents can be estimated from two empirical moments of the pass-rate distribution restricted to $\mathcal{Z}$. Since the kernel $w(p)=p^\alpha(1-p)^\beta$ normalized over $[0,1]$ yields a $\mathrm{Beta}(\alpha{+}1,\beta{+}1)$ density, we apply standard moment matching to this distribution:
\begin{equation}
  \frac{\alpha^*+1}{\alpha^*+\beta^*+2} = \bar{p}_{\mathcal{Z}}, \qquad \alpha^* + \beta^* = \frac{\bar{p}_{\mathcal{Z}}(1 - \bar{p}_{\mathcal{Z}})}{\text{Var}_{\mathcal{Z}}(p)} - 3
  \label{eq:exponent_selection}
\end{equation}
where $\bar{p}_{\mathcal{Z}}$ and $\text{Var}_{\mathcal{Z}}(p)$ are the mean and variance of $\{p_i\}_{i \in \mathcal{Z}}$. The kernel peak $p^* = \alpha^*/(\alpha^*+\beta^*)$ is approximately $\bar{p}_{\mathcal{Z}}$ for concentrated distributions. If the informative problems have pass rates concentrated around $0.4$ with low variance, the formula prescribes an asymmetric kernel ($\alpha^* < \beta^*$) peaked near $p^* \approx 0.4$; if they are spread broadly, it prescribes a flatter kernel (small $\alpha^* + \beta^*$). The formula requires no gradient computation---only the pass rates already computed for weighting. The following proposition makes this precise.
\begin{proposition}[\textbf{Data-Driven Exponent Selection via Moment Matching}]
\label{prop:exponent_selection}
Define the \emph{zone of proximal development} (ZPD) as $\mathcal{Z} = \{i : \epsilon \leq p_i \leq 1-\epsilon\}$ for cutoff $\epsilon > 0$ (e.g., $\epsilon = 1/K$), and let $P_{\mathcal{Z}}$ denote the restriction of the empirical pass-rate distribution $P$ to $\mathcal{Z}$, with mean $\bar{p}_{\mathcal{Z}} = \mathbb{E}_{P_{\mathcal{Z}}}[p]$ and variance $v_{\mathcal{Z}} = \text{Var}_{P_{\mathcal{Z}}}(p)$.

Since the kernel $w(p) = p^\alpha(1-p)^\beta$ normalized over $[0,1]$ yields a $\mathrm{Beta}(\alpha{+}1,\beta{+}1)$ density, the \emph{method-of-moments} exponents $(\alpha^*, \beta^*)$ are obtained by fitting $\mathrm{Beta}(\alpha{+}1,\beta{+}1)$ to the first two moments of $P_{\mathcal{Z}}$, i.e., $(\alpha{+}1)/(\alpha{+}\beta{+}2) = \bar{p}_{\mathcal{Z}}$ (normalized kernel mean $=$ data mean) and $\mathrm{Var}(\mathrm{Beta}(\alpha{+}1,\beta{+}1)) = v_{\mathcal{Z}}$:
\begin{equation}
  \frac{\alpha^*+1}{\alpha^*+\beta^*+2} = \bar{p}_{\mathcal{Z}}, \qquad \alpha^* + \beta^* = \frac{\bar{p}_{\mathcal{Z}}(1 - \bar{p}_{\mathcal{Z}})}{v_{\mathcal{Z}}} - 3
  \label{eq:moment_match}
\end{equation}
provided $v_{\mathcal{Z}} < \bar{p}_{\mathcal{Z}}(1 - \bar{p}_{\mathcal{Z}})/3$ (equivalently, $\alpha^* + \beta^* > 0$). Solving for individual exponents:
\begin{equation}
  \alpha^* = \bar{p}_{\mathcal{Z}}\!\left(\frac{\bar{p}_{\mathcal{Z}}(1 - \bar{p}_{\mathcal{Z}})}{v_{\mathcal{Z}}} - 1\right) - 1, \qquad \beta^* = (1 - \bar{p}_{\mathcal{Z}})\!\left(\frac{\bar{p}_{\mathcal{Z}}(1 - \bar{p}_{\mathcal{Z}})}{v_{\mathcal{Z}}} - 1\right) - 1
  \label{eq:alpha_beta_closed}
\end{equation}
The kernel peak at $p^* = \alpha^*/(\alpha^*+\beta^*)$ is approximately $\bar{p}_{\mathcal{Z}}$ for concentrated distributions (large $\alpha^*+\beta^*$), ensuring the kernel focuses on informative samples.  Moreover, the minimax robustness guarantee of Theorem~\ref{thm:minimax} continues to hold for the data-driven exponents: if the true SNR profile satisfies Assumption~\ref{asm:passrate_structure}(b$'$) with the fitted $(\alpha^*, \beta^*)$ in place of $(a', b')$, then pointwise worst-case efficiency is at least $\mathrm{sech}^2(\delta)$, with the same aggregate lower bound.
\end{proposition}

\begin{proof}
\textbf{Step 1: Design rationale.}
Theorem~\ref{thm:beta_optimal} establishes that the per-problem optimal weight takes the Beta kernel form $w(p) = C\,p^\alpha(1-p)^\beta$ but does not specify the exponents $(\alpha,\beta)$, which depend on the unknown SNR profile. A natural heuristic is to choose $(\alpha,\beta)$ so that the kernel concentrates its mass where the informative samples (those inside the ZPD) actually lie. This motivates matching the peak and spread of the kernel to the empirical distribution $P_{\mathcal{Z}}$ of pass rates within $\mathcal{Z}$.

Concretely, the kernel $w(p)=p^\alpha(1-p)^\beta$ normalized on $[0,1]$ has integral $B(\alpha{+}1,\beta{+}1)$, so the corresponding probability density is $\mathrm{Beta}(\alpha{+}1,\beta{+}1)$. We perform standard moment matching on this normalized kernel: let $a=\alpha{+}1$, $b=\beta{+}1$, and match the mean $a/(a{+}b) = \bar{p}_{\mathcal{Z}}$ and variance $ab/((a{+}b)^2(a{+}b{+}1)) = v_{\mathcal{Z}}$ of $\mathrm{Beta}(a,b)$ to the data moments.

\textbf{Step 2: Method-of-moments solution.}
With $a = \alpha+1$, $b=\beta+1$, we require:
\begin{align}
  \text{Mean matching:}\quad& \frac{a}{a+b} = \bar{p}_{\mathcal{Z}} \label{eq:mom1}\\
  \text{Variance:}\quad& \frac{ab}{(a+b)^2(a+b+1)} = v_{\mathcal{Z}} \label{eq:mom2}
\end{align}
From Eq.~\eqref{eq:mom1}: $b = a(1-\bar{p}_{\mathcal{Z}})/\bar{p}_{\mathcal{Z}}$. Define $s = a + b$. Then $a = s\bar{p}_{\mathcal{Z}}$, $b = s(1-\bar{p}_{\mathcal{Z}})$, and Eq.~\eqref{eq:mom2} gives:
\begin{equation}
  \frac{s^2 \bar{p}_{\mathcal{Z}}(1-\bar{p}_{\mathcal{Z}})}{s^2(s+1)} = v_{\mathcal{Z}} \quad\Longrightarrow\quad \frac{\bar{p}_{\mathcal{Z}}(1-\bar{p}_{\mathcal{Z}})}{s + 1} = v_{\mathcal{Z}} \quad\Longrightarrow\quad s = \frac{\bar{p}_{\mathcal{Z}}(1-\bar{p}_{\mathcal{Z}})}{v_{\mathcal{Z}}} - 1
\end{equation}
Converting back to kernel exponents: $\alpha^* = a-1 = s\,\bar{p}_{\mathcal{Z}} - 1 = \bar{p}_{\mathcal{Z}}\bigl(\frac{\bar{p}_{\mathcal{Z}}(1-\bar{p}_{\mathcal{Z}})}{v_{\mathcal{Z}}} - 1\bigr) - 1$ and $\beta^* = b-1 = s\,(1-\bar{p}_{\mathcal{Z}}) - 1 = (1-\bar{p}_{\mathcal{Z}})\bigl(\frac{\bar{p}_{\mathcal{Z}}(1-\bar{p}_{\mathcal{Z}})}{v_{\mathcal{Z}}} - 1\bigr) - 1$, yielding Eqs.~\eqref{eq:moment_match}--\eqref{eq:alpha_beta_closed}. The sum $\alpha^*+\beta^* = s-2 = \bar{p}_{\mathcal{Z}}(1-\bar{p}_{\mathcal{Z}})/v_{\mathcal{Z}} - 3$. The condition $\alpha^*+\beta^* > 0$ requires $v_{\mathcal{Z}} < \bar{p}_{\mathcal{Z}}(1 - \bar{p}_{\mathcal{Z}})/3$, i.e., the ZPD pass rates must be more concentrated than a uniform distribution ($v_{\text{Uniform}} = 1/12 = \bar{p}(1-\bar{p})/3$ for $\bar{p}=0.5$). When the data is exactly uniform, $s=2$ and $\alpha^*=\beta^*=0$, yielding the flat kernel $w(p)=1$; the default $\alpha=\beta=1$ reflects the theoretical prior from Theorem~\ref{thm:beta_optimal}, not data adaptation.

\textbf{Step 3: Robustness inheritance.}
Once $(\alpha^*, \beta^*)$ are selected, Theorem~\ref{thm:minimax} applies directly with $(a', b') = (\alpha^*, \beta^*)$: if the true SNR profile is within a multiplicative $e^{\pm\delta}$ of $p^{\alpha^*}(1-p)^{\beta^*}$, pointwise worst-case efficiency is $\mathrm{sech}^2(\delta) \geq 1 - \delta^2$, and the same value is an aggregate lower bound.

\textbf{Remark (Boundary with the default).}
When the ZPD pass-rate distribution is symmetric ($\bar{p}_{\mathcal{Z}} = 0.5$) with variance $v_{\mathcal{Z}} = 1/12$ (approximately uniform on $[0,1]$), we get $s = 0.25/(1/12) - 1 = 2$ and $\alpha^* = \beta^* = 0.5 \cdot 2 - 1 = 0$, yielding the flat kernel $w(p) = 1$. At $v_{\mathcal{Z}} = 1/20$ (more concentrated), the formula gives $s=4$, $\alpha^* = \beta^* = 0.5 \cdot 4 - 1 = 1$, recovering the default $w(p) = p(1-p)$. Thus the data-driven MoM reduces to the theory-based default when the ZPD distribution is moderately concentrated, and relaxes to uniform weighting when the data lacks clear structure.
\textbf{Remark (Practical interpretation).}
The formula has an intuitive reading:
\begin{itemize}
  \item The \emph{peak location} $p^* = \alpha^*/(\alpha^*+\beta^*) \approx \bar{p}_{\mathcal{Z}}$ (exact for $\bar{p}_{\mathcal{Z}}=0.5$) says: focus training where most of the informative problems are.
  \item The \emph{concentration} $\alpha^* + \beta^* = \bar{p}_{\mathcal{Z}}(1{-}\bar{p}_{\mathcal{Z}})/v_{\mathcal{Z}} - 3$ says: if informative problems are tightly clustered (small $v_{\mathcal{Z}}$), use a peaked kernel; if they are spread out (large $v_{\mathcal{Z}}$), use a broad kernel.
  \item The \emph{asymmetry} $\alpha^*/\beta^* \approx \bar{p}_{\mathcal{Z}}/(1-\bar{p}_{\mathcal{Z}})$ (for large $s$) says: if the student struggles ($\bar{p}_{\mathcal{Z}} < 0.5$), emphasize harder problems ($\alpha < \beta$); if the student is mostly competent ($\bar{p}_{\mathcal{Z}} > 0.5$), emphasize consolidation ($\alpha > \beta$).
\end{itemize}
\end{proof}

\section{Prompts and Implementation Details}
\subsection{Prompt Templates}
\label{app:prompts}
\begin{figure*}[h]
\centering
\begin{tcolorbox}[
  enhanced,
  colback=studentbg,
  colframe=orange!60!black,
  coltitle=orange!60!black,
  colbacktitle=studentbg,
  title={\textbf{Student Prompt}\quad $\pi_\theta(\cdot \mid x)$},
  fonttitle=\bfseries,
  rounded corners,
  boxrule=1.2pt,
  width=0.95\textwidth
]
\ttfamily\small
Solve the following math problem step by step. The last line of your response should be of the form Answer: \$Answer (without quotes) where \$Answer is the answer to the problem.\\[6pt]
Find all real numbers $x$ such that $x^3 - 6x^2 + 11x - 6 = 0$.\\[4pt]
Remember to put your answer on its own line after ``Answer:''.
\end{tcolorbox}
\vspace{6pt}
\begin{tcolorbox}[
  enhanced,
  colback=teacherbg,
  colframe=green!60!black,
  coltitle=green!60!black,
  colbacktitle=teacherbg,
  title={\textbf{Teacher Prompt}\quad $\pi_{\bar{\theta}}(\cdot \mid x,\; y_{\mathcal{E}})$},
  fonttitle=\bfseries,
  rounded corners,
  boxrule=1.2pt,
  width=0.95\textwidth
]
\ttfamily\small
Find all real numbers $x$ such that $x^3 - 6x^2 + 11x - 6 = 0$.\\[6pt]
\textcolor{green!40!black}{\textbf{Expert solution:}} \{expert solution\}. Treat it as guidance: understand the reasoning and then write the solution in your own words. Do not copy the original answer verbatim.\\[4pt]
Solve the following math problem step by step. The last line of your response should be of the form Answer: \$Answer (without quotes) where \$Answer is the answer to the problem.\\[6pt]
Find all real numbers $x$ such that $x^3 - 6x^2 + 11x - 6 = 0$.\\[4pt]
Remember to put your answer on its own line after ``Answer:''.
\end{tcolorbox}
\caption{\textbf{Prompt example for student and teacher policies.} Both policies share the same model family but differ in conditioning context. The teacher receives the expert solution $y_{\mathcal{E}}$ as additional context, while the student receives only the original problem. This contextual asymmetry enables black-box expert guidance to be transferred into white-box teacher logits for distillation.}
\label{fig:prompts}
\end{figure*}

\subsection{Implementation Details and Hyperparameters}
\label{app:hyperparameters}
\paragraph{Implementation and data.}
All training and evaluation data used in this work are publicly available. Our training pipeline builds on a publicly available on-policy self-distillation codebase released by the same team and the verl distributed-training framework~\citep{sheng2024hybridflow}. Some internal orchestration and infrastructure-specific code cannot be released due to company policy, but the method, algorithmic changes, and full experimental configuration are described in this paper and appendix.

We partition DAPO-Math-17k into two disjoint prompt splits, one used for the Qwen3 distillation track and one for the self-distillation track. For pass-rate estimation and evaluation, each completion is reduced to its final extracted answer and compared against the benchmark reference after lightweight normalization (e.g., stripping whitespace, delimiters, and equivalent formatting).

The Hard Filter baseline uses the same pass-rate estimates as \textsc{Paced}, but replaces the Beta kernel with a binary keep/drop rule: a problem is retained iff $0.2 \leq p \leq 0.8$, and dropped otherwise. The same two-sided thresholds are used in both training settings. With the default $K{=}8$ rollouts, this corresponds to keeping problems with 2 through 6 correct rollouts out of 8, and dropping problems with 0, 1, 7, or 8 correct rollouts.

\paragraph{Teacher preparation (Qwen3-8B\textsubscript{GRPO}).}
The frozen teacher used in the distillation track is Qwen3-8B fine-tuned with GRPO~\citep{shao2024deepseekmath} on the distillation split of DAPO-Math-17k~\citep{yu2025dapo}.
Concretely, we run GRPO with group size $G{=}8$, KL penalty coefficient $\beta_{\mathrm{KL}}{=}0.001$, learning rate $1{\times}10^{-6}$, global batch size 128, and a cosine schedule over 2 epochs; all other settings follow the DAPO recipe~\citep{yu2025dapo}.
The resulting model serves as a \emph{frozen} teacher throughout all distillation experiments; its weights are never updated during student training.
This GRPO fine-tuning step is a one-time offline cost and is not part of the \textsc{Paced} training loop itself.

In the forward-KL track, the external expert first produces a reference solution, after which the frozen Qwen3-8B\textsubscript{GRPO} teacher regenerates a target response conditioned on the original problem and the expert solution. Teacher regeneration uses the same prompt template family as Appendix~\ref{app:prompts}; student rollouts for pass-rate estimation use temperature 1.0. Reasoning-benchmark evaluation uses normalized final-answer matching, whereas MMLU retention is measured via \texttt{lm-evaluation-harness}~\citep{eval-harness} with 5-shot prompting. For two-stage experiments, the total number of optimization steps is matched to the corresponding single-stage runs, with the budget split across stages as specified in Table~\ref{tab:twostage_bridge} and Table~\ref{tab:twostage_stage1_budget}.

\paragraph{Hyperparameters.}
Table~\ref{tab:hyperparameters} summarizes the common configuration shared across runs, with setting-specific differences called out explicitly when needed.

\begin{table}[ht]
\renewcommand{\arraystretch}{1.15}
\centering
\setlength{\tabcolsep}{6pt}
\small
\begin{tabular}{ll}
  \toprule
  \textbf{Parameter} & \textbf{Value} \\
  \midrule
  \textbf{General} & \\
  Models & Qwen2.5-Math-7B-Instruct (self-distillation), Qwen3-1.7B (teacher: Qwen3-8B\textsubscript{GRPO}) \\
  \midrule
  \textbf{Data} & \\
  Training prompts & DAPO-Math-17k~\citep{yu2025dapo} \\
  Max prompt length (student) & 1{,}024 tokens (problem only) \\
  Max prompt length (teacher) & 3{,}072 tokens (problem + expert solution) \\
  Max response length & 16{,}384 tokens (training) \\
  \midrule
  \textbf{Generation (student rollout)} & \\
  Temperature & 1.0 \\
  Rollouts per prompt ($K$) & 8 \\
  Max generation tokens & 8{,}192 \\
  \midrule
  \textbf{Evaluation} & \\
  Benchmarks & MATH-500, AIME 2024, AIME 2025, MMLU (2{,}000-question random subsample) \\
  Metric & 8-sample mean accuracy (\%) \\
  Temperature & 0.6 \\
  Top-$p$ & 0.95 \\
  Rollouts per prompt & 8 \\
  Max generation tokens & 30{,}000 \\
  Eval frequency & Every 10 steps \\
  \midrule
  \textbf{Training} & \\
  Optimizer & AdamW \\
  Learning rate & $1 \times 10^{-7}$, constant \\
  Weight decay & 0.01 \\
  Gradient clipping & 1.0 (max norm) \\
  Global batch size & 32 \\
  Micro-batch size per GPU & 2 \\
  Epochs & 2 \\
  Precision & bfloat16 \\
  \midrule
  \textbf{Infrastructure} & \\
  GPUs & $8 \times$ NVIDIA H200 \\
  Tensor parallelism (inference) & 2 \\
  Sequence parallelism (training) & Ulysses, degree 8 \\
  FSDP parameter offload & Enabled \\
  FSDP optimizer offload & Enabled \\
  Gradient checkpointing & Enabled \\
  \bottomrule
\end{tabular}
\caption{Hyperparameters for \textsc{Paced}. Shared settings are listed once, with setting-specific differences noted explicitly.}
\label{tab:hyperparameters}
\end{table}
\clearpage

\section{Additional Experiments}
\label{app:additional_experiments}

This section collects the supplementary empirical analyses referenced throughout the main text: sensitivity checks (rollout count, recomputation frequency, and two-stage budget split), mechanistic evidence (curriculum evolution and empirical SNR), a detailed comparison with the AKL baseline, and cross-family generalization to Llama.

\subsection{Sensitivity Analysis}
\subsubsection{Effect of Weight Exponents}
\label{app:ablation_exponents}

We first examine the most direct design choice in \textsc{Paced}: the shape of the pass-rate kernel. Unless otherwise stated, ablations in this subsection use Qwen3-1.7B with forward KL as the base distillation loss.

\begin{table}[h]
\caption{Ablation on pass-rate weight exponents $w(p)=p^\alpha(1-p)^\beta$ using forward KL divergence as the distillation loss (Qwen3-1.7B).}
\centering
\begin{tabular}{cc|cc}
\toprule
$\alpha$ & $\beta$ & MATH-500 ($\uparrow$) & Forgetting on MMLU ($\downarrow$) \\
\midrule
1 & 1 & \mstd{79.4}{} & 1.4\% \\
1 & 2 & \mstd{80.8}{} & 2.1\% \\
2 & 1 & \mstd{76.9}{} & 3.0\% \\
1 & 3 & \mstd{80.3}{} & 3.6\% \\
3 & 1 & \mstd{75.7}{} & 2.9\% \\
\bottomrule
\end{tabular}
\label{tab:ablation_exponents}
\end{table}

\textbf{Interpretation.} Tilting the kernel toward harder problems ($\beta > \alpha$) improves MATH-500 up to a point: $(\alpha{=}1,\,\beta{=}2)$ yields the best score ($80.8\%$) but raises forgetting to $2.1\%$, whereas kernels favoring easier problems ($\alpha > \beta$) degrade both metrics. The default $(\alpha{=}\beta{=}1)$ remains the best plasticity--stability balance.

\subsubsection{Sensitivity to Number of Rollouts $K$}
\label{app:ablation_K}
The pass-rate estimate $\hat{p}_i = (\text{\# correct out of }K)$ controls the Beta kernel weights.
We ablate $K \in \{4, 8, 16\}$ on Qwen3-1.7B distillation (forward KL, $\alpha{=}\beta{=}1$) to test (i)~how estimation noise from small $K$ affects final performance, (ii)~whether large $K$ yields further gains, and (iii)~the associated compute cost.

\begin{table}[h]
\caption{Sensitivity to number of rollouts $K$ for pass-rate estimation. All results use Qwen3-1.7B with forward KL and default exponents $(\alpha{=}\beta{=}1)$.}
\centering
\begin{tabular}{c|cccc}
\toprule
$K$ & MATH-500 ($\uparrow$) & AIME 2024 ($\uparrow$) & AIME 2025 ($\uparrow$) & MMLU Fgt.\ ($\downarrow$) \\
\midrule
4  & \mstd{78.0}{} & \mstd{23.3}{} & \mstd{19.5}{} & 1.7\% \\
8  & \mstd{79.4}{} & \mstd{25.1}{} & \mstd{20.6}{} & 1.4\% \\
16 & \mstd{80.1}{} & \mstd{26.3}{} & \mstd{21.8}{} & 1.5\% \\
\bottomrule
\end{tabular}
\label{tab:ablation_K}
\end{table}

\textbf{Interpretation.} Halving the rollout budget to $K{=}4$ costs $1.4$ points on MATH-500 and $1.1$ on AIME~2025, while forgetting increases slightly to $1.7\%$. This confirms that the Beta kernel's smooth weighting is robust to the noisier pass-rate estimates from small $K$---unlike hard-threshold filters, a continuous weight function does not amplify estimation errors near the decision boundary. Doubling to $K{=}16$ yields modest gains ($+0.7$ MATH-500, $+1.2$ AIME~2025) with diminishing returns, suggesting $K{=}8$ strikes a practical balance between estimation quality and rollout cost.

\subsubsection{Effect of Periodic Pass-Rate Recomputation}
\label{app:ablation_recompute}

The main experiments estimate pass rates once before training (single-pass). We ablate the recomputation interval on Qwen3-1.7B distillation (forward KL, $\alpha{=}\beta{=}1$, $K{=}8$).

\begin{table}[h]
\caption{Effect of periodic pass-rate recomputation on Qwen3-1.7B (forward KL, $\alpha{=}\beta{=}1$). ``Single-pass'' estimates pass rates once before training; ``Every $N$ steps'' recomputes pass rates and updates weights at the specified interval.}
\centering
\begin{tabular}{l|cccc}
\toprule
\textbf{Recompute interval} & MATH-500 ($\uparrow$) & AIME 2024 ($\uparrow$) & AIME 2025 ($\uparrow$) & MMLU Fgt.\ ($\downarrow$) \\
\midrule
Single-pass & \mstd{79.4}{} & \mstd{25.1}{} & \mstd{20.6}{} & 1.4\% \\
Every 200 steps & \mstd{79.9}{} & \mstd{26.6}{} & \mstd{22.5}{} & 1.6\% \\
Every 100 steps & \mstd{80.2}{} & \mstd{26.9}{} & \mstd{23.9}{} & 1.4\% \\
Every 50 steps  & \mstd{80.4}{} & \mstd{28.5}{} & \mstd{23.3}{} & 1.5\% \\
\bottomrule
\end{tabular}
\label{tab:ablation_recompute}
\end{table}

\textbf{Interpretation.} Periodic recomputation yields modest but consistent gains: the best interval (every 50 steps) improves over single-pass by $+1.0$ on MATH-500, $+3.4$ on AIME~2024, and $+2.7$ on AIME~2025, while forgetting remains low ($1.5\%$). The Beta kernel's continuous shape provides a natural buffer against stale pass rates: unlike hard filters, the smooth function $w(p) = p(1-p)$ absorbs pass-rate drift gracefully. This robustness is precisely what Theorem~\ref{thm:minimax} quantifies: bounded misspecification incurs only $O(\delta^2)$ efficiency loss.

\subsubsection{Two-Stage KL Schedule: Full Results and Budget Ablation}
\label{app:twostage_budget}
\begin{table}[h]
\caption{Two-stage order comparison on Qwen3 with the same pass-rate weighting $w(p)=p(1-p)$. Pass rates are recomputed once between stages; the first half of training steps uses Stage~1 and the second half uses Stage~2. Results are reported as 8-sample mean accuracy. The first two rows give the corresponding single-loss references under the same midpoint-recompute setup, and the last two rows isolate schedule order.}
\centering
\begin{tabular}{ll|cccc}
\toprule
\textbf{Stage 1} & \textbf{Stage 2} & MATH-500 ($\uparrow$) & AIME 2024 ($\uparrow$) & AIME 2025 ($\uparrow$) & MMLU Fgt.\ ($\downarrow$) \\
\midrule
Paced KL & Paced KL & \mstd{79.7}{} & \mstd{25.6}{} & \mstd{21.1}{} & 1.3\% \\
Paced RevKL & Paced RevKL & \mstd{78.8}{} & \mstd{23.5}{} & \mstd{19.4}{} & 1.2\% \\
Paced RevKL & Paced KL & \mstd{76.9}{} & \mstd{23.0}{} & \mstd{18.2}{} & 2.5\% \\
\textbf{Paced KL} & \textbf{Paced RevKL} & \textbf{\mstd{81.4}{}} & \textbf{\mstd{26.1}{}} & \textbf{\mstd{22.8}{}} & \textbf{1.1\%} \\
\bottomrule
\end{tabular}
\label{tab:twostage_bridge}
\end{table}

KL $\rightarrow$ RevKL improves over single-loss Paced KL by $+1.7/+0.5/+1.7$ on MATH-500/AIME~2024/AIME~2025, while the reversed order (RevKL $\rightarrow$ KL) underperforms both single-loss references and incurs higher forgetting ($2.5\%$)---confirming that mode-coverage must precede consolidation. The 50/50 budget split (Table~\ref{tab:twostage_stage1_budget} below) gives the strongest overall result.

\begin{table}[h]
\caption{Ablation on the fraction of training steps allocated to Stage~1 for two-stage distillation on Qwen3 under the KL $\rightarrow$ RevKL schedule. The first $x$\% of steps use Paced KL (Stage~1) and the remaining steps use Paced RevKL (Stage~2). Results are 8-sample mean accuracy.}
\centering
\small
\setlength{\tabcolsep}{4pt}
\begin{tabular}{lc|cccc}
\toprule
\textbf{Schedule} & \textbf{Stage 1 ratio} & MATH-500 ($\uparrow$) & AIME 2024 ($\uparrow$) & AIME 2025 ($\uparrow$) & MMLU Fgt.\ ($\downarrow$) \\
\midrule
KL $\rightarrow$ RevKL & 25\% & \mstd{78.9}{} & \mstd{22.4}{} & \mstd{19.3}{} & 1.6\% \\
KL $\rightarrow$ RevKL & 50\% & \mstd{81.4}{} & \mstd{26.1}{} & \mstd{22.8}{} & 1.1\% \\
KL $\rightarrow$ RevKL & 75\% & \mstd{80.1}{} & \mstd{26.9}{} & \mstd{20.6}{} & 1.1\% \\
\bottomrule
\end{tabular}
\label{tab:twostage_stage1_budget}
\end{table}
The 50/50 split offers the best overall trade-off ($81.4\%$ MATH-500, $26.1\%$ AIME~2024, $22.8\%$ AIME~2025), achieving the strongest MATH-500 and AIME~2025 performance while matching the lowest forgetting, which suggests that equal allocation between mode-covering and mode-seeking stages strikes the right balance.

\subsection{Mechanistic Validation}

\subsubsection{Curriculum Progression}

\label{app:curriculum}
We track how the pass-rate distribution evolves during Qwen3-1.7B distillation (forward KL, $\alpha{=}\beta{=}1$, $K{=}8$). At each checkpoint, we re-evaluate pass rates on the full training set and report the fraction of problems in three bins.

Table~\ref{tab:passrate_evolution} traces the migration of problems through the difficulty landscape during training.
As the student strengthens, problems flow steadily from the ``too hard'' regime ($p<0.2$) through the zone of proximal development ($p\in[0.2,0.8]$) and into the ``mastered'' side ($p>0.8$): the fraction with $p>0.8$ grows from $32\%$ to $74\%$ over 300 steps, while the average pass rate $\bar p$ rises monotonically from $0.61$ to $0.84$.
Notably, the Med-$p$ bin shrinks from $51\%$ to $21\%$, indicating that the pool of maximally informative problems is gradually depleted as the student masters more of the curriculum.
This progressive depletion has a practical implication: the effective training signal weakens over time as fewer problems remain in the ZPD, which is consistent with the diminishing marginal returns typical of later training stages and naturally favors more consolidative objectives (e.g., reverse-KL behavior) once the ZPD has substantially contracted.
The low-$p$ tail also shrinks (from $17\%$ to $5\%$), indicating that previously intractable problems gradually become tractable.

\begin{table}[h]
\caption{Evolution of the pass-rate distribution and average pass rate $\bar p$ across training. The distillation signal peaks when most problems enter the $p\in[0.2,0.8]$ zone.}
\centering
\begin{tabular}{l|cccc}
\toprule
\textbf{Training Stage} & Low $p$ ($<\!0.2$) & Med $p$ ($0.2$--$0.8$) & High $p$ ($>\!0.8$) & Avg pass rate $\bar p$ \\
\midrule
Step 0 (Init) & 17\% & 51\% & 32\% & 0.61 \\
Step 100     & 12\% & 32\% & 56\% & 0.70 \\
Step 200     & 9\% & 24\% & 67\% & 0.78 \\
Step 300 & 5\% & 21\% & 74\% & 0.84 \\
\bottomrule
\end{tabular}
\label{tab:passrate_evolution}
\end{table}

\subsection{Detailed AKL Baseline Comparison}
\label{app:akl_analysis}

AKL~\citep{wu2025akl} is a strong baseline that adapts the distillation signal dynamically, but at a fundamentally different granularity: it modulates the KL coefficient \emph{per token} based on teacher--student logit discrepancy, whereas \textsc{Paced} modulates \emph{per problem} based on pass rate.
The performance gap reflects a \textbf{structural} difference between token-level and problem-level adaptation.
AKL adjusts \emph{how much} the student learns from each token within a given problem, but treats all problems equally---an intractable problem ($p \approx 0$) receives the same total training budget as a productive one ($p \approx 0.5$).
This means AKL cannot suppress the noisy, high-variance gradients from intractable problems or the redundant gradients from mastered ones; it only rebalances \emph{within} each problem.
In contrast, \textsc{Paced} operates at the problem level via a continuous Beta kernel $w(p)=p^\alpha(1{-}p)^\beta$, concentrating the entire training budget on problems where the student has partial competence.

Regarding forgetting, AKL's per-token adaptation implicitly down-weights tokens where the teacher--student gap is extreme, which partially mitigates catastrophic forgetting. However, \textsc{Paced} still achieves lower or comparable forgetting in both tracks (Tables~\ref{tab:forgetting_distill}--\ref{tab:forgetting_self}). The difference is that AKL cannot suppress entire intractable problems: even with per-token adaptation, passing gradients through a $p \approx 0$ problem injects noise that accumulates across tokens.

Notably, the two approaches are \emph{orthogonal} and could in principle be combined: \textsc{Paced} selects \emph{which} problems to train on, while AKL optimizes \emph{how} to train on each selected problem.

\subsection{Cross-Family Generalization: Llama-3.1-8B-Instruct}
\label{app:llama}

To verify that \textsc{Paced} is not specific to the Qwen model family, we replicate the distillation experiment using Llama-3.3-70B-Instruct~\citep{grattafiori2024llama3} as teacher and Llama-3.1-8B-Instruct as student, with forward KL as the base loss. All other settings (DAPO training data, $K{=}8$ rollouts, $\alpha{=}\beta{=}1$) follow the Qwen3 distillation track, with the same learning rate of $1 \times 10^{-7}$.

\begin{table}[h]
\caption{Distillation from Llama-3.3-70B-Instruct to Llama-3.1-8B-Instruct (forward KL family): reasoning performance (8-sample mean accuracy) and retention.}
\centering
\begin{tabular}{lcccc}
\toprule
\textbf{Method} & \textbf{MATH-500} ($\uparrow$) & \textbf{AIME 2024} ($\uparrow$) & \textbf{AIME 2025} ($\uparrow$) & \textbf{MMLU Fgt.} ($\downarrow$) \\
\midrule 
Forward KL (unweighted) & \mstd{72.3}{} & \mstd{22.9}{} & \mstd{5.1}{} & 3.5\% \\
Hard Filter Forward KL & \mstd{75.2}{} & \mstd{25.2}{} & \mstd{7.8}{} & 1.5\% \\
\textbf{\textsc{Paced} Forward KL} & \textbf{\mstd{76.7}{}} & \textbf{\mstd{27.7}{}} & \textbf{\mstd{9.2}{}} & \textbf{1.4\%} \\
\bottomrule
\end{tabular}
\label{tab:llama_self}
\end{table}

\textbf{Interpretation.} The pattern observed on Qwen transfers to a different model family: \textsc{Paced} improves MATH-500 by $+4.4$ and AIME 2024/2025 by $+4.8/+4.1$ over unweighted forward KL, while reducing forgetting from $3.5\%$ to $1.4\%$. The gains over hard filtering ($+1.5/+2.5/+1.4$) confirm that smooth Beta-kernel weighting extracts more signal than binary thresholding, consistent with the Qwen results.

\section{Additional Interpretations}
\label{app:connections}
\label{app:ib_connection}

The full pipeline can be viewed informally as a cascaded information bottleneck~\citep{tishby2000ib}:
\begin{equation}
  Y_{\mathcal{E}} \xrightarrow{\text{reference generation}} Y_T \xrightarrow{\text{pass-rate weighting}} w(p) \cdot Y_T \xrightarrow{\text{distillation}} \theta_{\text{updated}},
\end{equation}
where (i) reference generation lets the teacher re-express expert solutions in its own distributional voice, (ii) pass-rate weighting down-weights problems with low learning signal via $w(p) = p^\alpha(1-p)^\beta$, and (iii) distillation transfers knowledge from teacher to student via the chosen loss function. This view is purely interpretive and not used in our formal guarantees.
\begin{remark}[\textbf{Noise Filtering Interpretation}]
\label{rem:cascade}
At extreme pass rates, teacher-generated responses may carry teacher-specific artifacts, and $w(p) \to 0$ as $p \to 0$ or $p \to 1$ suppresses these noisy regimes. At intermediate pass rates, the student has sufficient capacity to extract transferable knowledge without memorizing artifacts, so $w(p) = p(1-p)$ naturally focuses training on the student's zone of proximal development, qualitatively resembling an information-bottleneck-style noise filter~\citep{tishby2000ib}.
\end{remark}

\begin{remark}[\textbf{Connection to Fisher Information}]
\label{rem:fisher}
The pass rate $p$ can be viewed as the parameter of a Bernoulli random variable (correct/incorrect) with Fisher information $\mathcal{I}(p) = 1/(p(1-p))$. The inverse Fisher information $p(1-p)$ is exactly our default weight ($\alpha = \beta = 1$), and the generalization $p^\alpha(1-p)^\beta$ allows asymmetric emphasis when practitioners wish to prioritize harder or easier problems.
\end{remark}

\begin{remark}[\textbf{Geometric Interpretation}]
\label{rem:geometric}
Let $\mathcal{M}_\theta$ denote the student's representational manifold. For teacher responses at low pass rates, $y_T$ is partially off-manifold and gradients contain orthogonal components that enable acquiring new capabilities; at high pass rates, $y_T$ is nearly on-manifold and gradients are predominantly tangential, refining existing skills. The pass-rate kernel $w(p) = p(1-p)$ scales both regimes, suppressing large off-manifold steps when $p \to 0$ and unnecessary tangential steps when $p \to 1$.
\end{remark}
\end{document}